\documentclass{article}

\usepackage{microtype}
\usepackage[dvipdfm]{graphicx}
\usepackage{bmpsize}
\usepackage{subfigure}
\usepackage{booktabs} 

\usepackage{hyperref}

\usepackage[accepted]{icml2025}

\usepackage[utf8]{inputenc} 
\usepackage[T1]{fontenc}    
\usepackage{enumitem}

\usepackage{amsmath}
\usepackage{amssymb}
\usepackage{amsthm}
\usepackage{amsfonts}
\usepackage{mathtools}
\usepackage{bbm}

\usepackage[capitalize,noabbrev]{cleveref}

\usepackage{multirow}
\usepackage{etoolbox,xspace}
\usepackage{pifont}

\newcommand{\RR}{\mathbb{R}}

\newcommand{\cD}{\mathcal{D}}

\newcommand{\cL}{\mathcal{L}}

\newcommand\norm[1]{\left\lVert#1\right\rVert}

\newcommand{\bp}{\mathbf{p}}

\newcommand{\bk}{\mathbf{k}}

\newcommand{\bx}{\mathbf{x}}

\newcommand{\by}{\mathbf{y}}
\newcommand{\bv}{\mathbf{v}}

\newcommand{\bq}{\mathbf{q}}

\newcommand{\btheta}{\boldsymbol{\theta}}

\newcommand{\ours}{LoRSU\xspace}

\newcommand{\xmark}{\ding{55}}%
\newcommand\mtiny[1]{\mbox{\tiny\ensuremath{#1}}}

\theoremstyle{plain}
\newtheorem{theorem}{Theorem}[section]

\newtheorem{lemma}[theorem]{Lemma}
\newtheorem{corollary}[theorem]{Corollary}
\theoremstyle{definition}
\newtheorem{definition}[theorem]{Definition}

\theoremstyle{remark}
\newtheorem{remark}[theorem]{Remark}

\DeclareMathOperator*{\argmax}{arg\,max}

\icmltitlerunning{Efficient Few-Shot Continual Learning in Vision-Language Models}

\begin{document}

\twocolumn[
\icmltitle{Efficient Few-Shot Continual Learning in Vision-Language Models}


\icmlsetsymbol{equal}{*}

\begin{icmlauthorlist}
\icmlauthor{Aristeidis Panos}{cam}
\icmlauthor{Rahaf Aljundi}{comp}
\icmlauthor{Daniel Olmeda Reino}{comp}
\icmlauthor{Richard E. Turner}{cam}
\end{icmlauthorlist}

\icmlaffiliation{cam}{University of Cambridge}
\icmlaffiliation{comp}{Toyota Motor Europe}

\icmlcorrespondingauthor{Aristeidis Panos}{ap2313@cam.ac.uk}

\icmlkeywords{Machine Learning, ICML}

\vskip 0.3in
]



\printAffiliationsAndNotice{}  

\begin{abstract}
Vision-language models (VLMs) excel in tasks such as visual question answering and image captioning. However, VLMs are often limited by their use of pretrained image encoders, like CLIP, leading to image understanding errors that hinder overall performance. On top of that, real-world applications often require the model to be continuously adapted as new and often limited data continuously arrive. 
To address this, we propose LoRSU (Low-Rank Adaptation with Structured Updates), a robust and computationally efficient method for selectively updating image encoders within VLMs. LoRSU introduces structured and localized parameter updates, effectively correcting performance on previously error-prone data while preserving the model’s general robustness. Our approach leverages theoretical insights to identify and update only the most critical parameters, achieving significant resource efficiency. Specifically, we demonstrate that LoRSU reduces computational overhead by over 25x compared to full VLM updates, without sacrificing performance. Experimental results on VQA tasks in the few-shot continual learning setting, validate LoRSU’s scalability, efficiency, and effectiveness, making it a compelling solution for image encoder adaptation in resource-constrained environments.
\end{abstract}
\section{Introduction}\label{sec:intro}

Large Language Models (LLMs) have revolutionized natural language understanding and generation, enabling significant advancements across diverse applications. As intelligent agents are increasingly expected to operate in real-world multimodal environments, integrating visual understanding becomes essential. Vision-Language Models (VLMs) extend LLMs by incorporating visual information, either through pre-trained vision encoders or end-to-end multimodal training. These models have demonstrated state-of-the-art performance on vision-language tasks such as visual question answering (VQA) and image captioning, highlighting their potential for general-purpose multimodal reasoning \cite{chen2024internvl,Qwen2VL}.

Approaches that rely on pretrained image encoders typically use variants of the CLIP model \citep{radford2021learning}, which is kept frozen in the vision-language binding process~\cite{liu2024visual}. CLIP is a widely deployed vision transformer that has strong zero-shot capabilities in various tasks and domains. However several  existing works have highlighted various weaknesses of CLIP on out of domain data~\cite{liu2024visual,zhu2023minigpt,chen2023minigpt,li2023blip,tong2024eyes}. When deploying VLMs as visual assistants in new domains, it is then expected that VLMs can be updated using a few images gathered from the target environment whenever deficiencies are noted. 

Continual learning allows a model to be continuously updated as new data from new tasks or domains are encountered.  Recent literature on continual learning (CL) of vision language models focus on updating either the LLM~\cite{srivastava2024improving} or language projection layers~\cite{das2024one}, maintaining a frozen image encoder. 

In vision language models, the LLM component provides reasoning and factual knowledge, while the image encoder's role is to extract robust and accurate visual features. In this work, we argue that adapting VLMs to new visual domains or tasks is more effective and efficient when the image encoder is updated rather than the LLM. Figure~\ref{fig:dalle_tsi} highlights this issue using images from the Toyota Smart Home dataset (TSI)~\cite{das2019toyota} dataset: in the first column, LLaVA~\cite{liu2024llava} struggles to recognize the person's action in the original image but accurately describes the same action in a generated image from OpenAI's DALL·E 2. This example underscores that the visual shift, rather than the LLM's understanding of the action, is the main source of weakness.

Motivated by the above limitations, we introduce a novel parameter-efficient fine-tuning (PEFT) method called \ours (Low-Rank Adaptation with Structured Updates) for selectively updating specific modules of the transformer blocks of image encoders within VLMs. The right column of Figure~\ref{fig:dalle_tsi} 
illustrates the (correct) responses of LLaVA after updating the image encoder separately with our method on a low-number of samples from TSI dataset compared to the pretrained LLaVA's (wrong) response.

Through extensive experiments, we demonstrate that updating the image encoder is essential for improving the performance of the VLM that relies on it. More importantly, this approach is computationally efficient, as the image encoder has significantly fewer parameters compared to the language model and the method is less prone to forgetting, especially the LLM knowledge.

We evaluated our approach on various VQA tasks comparing to state-of-the-art CL methods and the PEFT baseline LoRA\cite{hu2021lora} on various few-shot CL settings. We show significant improvements of the full VLM model on all settings and very low rates of forgetting  without using any replay buffer of data from the previous tasks. By selectively updating the image encoder, our method provides a robust and efficient solution for handling visual shifts. This targeted adaptation strategy avoids the need to modify the entire model, preserving existing knowledge whilst ensuring strong performance in new domains.

The contributions of the paper are as follows:
\begin{itemize}[noitemsep,topsep=1pt,parsep=1pt,partopsep=1pt]
    \item We propose LoRSU, a novel replay-free PEFT method tailored for few-shot continual learning.
    \item We introduce two new VQA datasets, TSI and DALLE, created to expose the limitations of pre-trained image encoders in VLMs.
    \item We conduct the first large-scale study of few-shot CL in VLMs, evaluating \ours across ten diverse VQA datasets and benchmarking against state-of-the-art PEFT and CL methods. \ours consistently outperforms all baselines.
\end{itemize}

\begin{figure}[t]
\vskip 0.2in
    \centering   \includegraphics[width=1.0\linewidth,height=4cm]{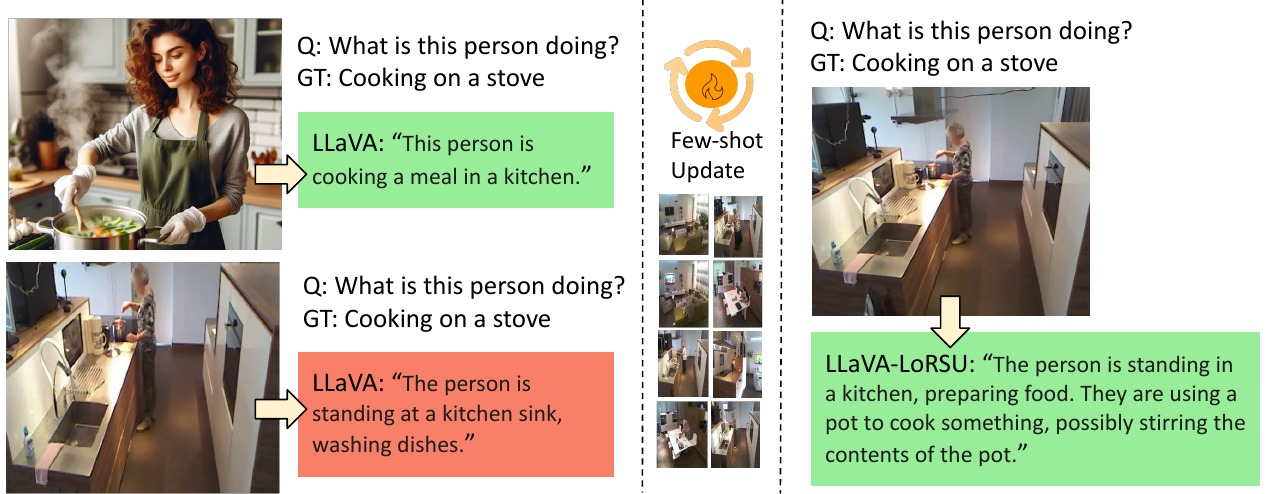}
\caption{(Left) Responses of the pretrained LLaVA to samples from  TSI dataset (bottom) compared to DALL·E 2 generated images (top) for the \textit{`cooking on a stove'} class. (Right) LLaVA’s correct response to the same TSI image after fine-tuning LLaVA using LoRSU.}
\label{fig:dalle_tsi}
\vskip -0.2in
\end{figure}

\section{Related Work}\label{sec:related_work}

\begin{figure*}
\vskip 0.2in
\centering
\includegraphics[width=0.8\textwidth]{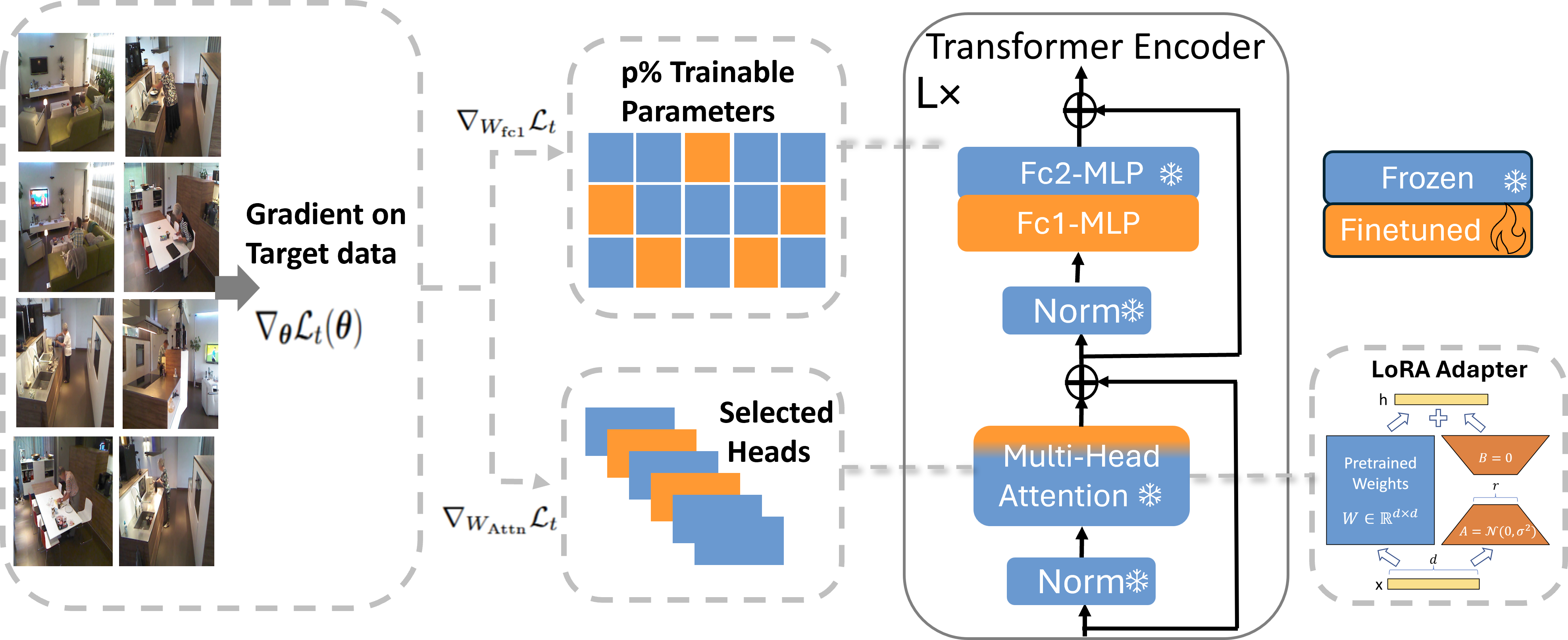}
\caption{LoRSU mechanism: After computing the gradient $\nabla_{\btheta} \cL_t (\btheta)$ over the target dataset at time $t$, LoRSU picks a small number of attention heads and a small number of paremeters from the first linear layer of the MLP module in the transformer block based on the magnitude of the gradients of $\nabla_{W_{\text{Attn}}} \cL_t $ and $\nabla_{W_{\text{fc1}}} \cL_t$, respectively. Computational efficiency is ensured by introducing LoRA adapters to the attention weight matrices.}
\label{fig:LorSU}
\vskip -0.2in
\end{figure*}
\looseness=-1\textbf{Continual Learning.} Our work falls within the continual learning literature, where a model needs to be updated incrementally as new data arrive, accumulating knowledge over tasks and reducing forgetting of previously acquired information \cite{de2021continual}.

\textbf{Continual Learning for Multimodal Language Models.} \citet{wu2024continual} provide a survey on continual learning for LLMs, highlighting challenges of computational efficiency and forgetting. \citet{srivastava2024improving} explored continual multi-modal learning on VQA datasets, keeping the vision encoder frozen. \citet{he2023continual} examined continual instruction tuning with sequential VQA datasets, proposing a method where the projection head is expanded for each new task. \citet{das2024one} introduced a pseudo-rehearsal strategy for vision-language models, updating only the language projection layer. Our method adapts only the vision encoder, preserving language capabilities.

\textbf{Continual Learning with Few-Shot Updates.} \citet{verwimp2023continual} posits that an ideal continual learning solution would enable continual correction of model's mistakes at a lower computational cost than retraining from scratch. However, most continual few-shot learning from pre-trained models focuses on classification tasks and introduces solutions that cannot scale to large multimodal models. \citet{panos2023first} update the vision encoder on the first task only, later adapting a covariance matrix for incoming tasks. \citet{goswami2024calibrating} calibrate the covariance matrix for new classes based on semantic similarity. \citet{zhao2024safe} introduce few and slow updates, proposing a transfer loss function and a cross-classification loss to mitigate catastrophic forgetting. Few-shot updates can also be viewed through the lens of model editing~\cite{Sinitsin2020Editable}. MEND~\cite{mitchell2022fast} scales model editing to large language models by transforming the gradient obtained from fine-tuning, through a low-rank decomposition fed to auxiliary networks designed to make fast, local edits to a pre-trained model, requiring a set of unrelated examples to prevent forgetting. ROME~\cite{meng2022locating} applies causal tracing to identify layers where incorrect factual knowledge is stored, applying a low-rank update. However, ROME does not scale to continual updates or non-association types of updates. \citet{cheng2024edit} studied multi-modal editing, showing negligible deterioration in multi-modal task performance when updating language models but severe forgetting when updating vision encoders. To the contrary, our method focuses on adapting the vision encoder rather than updating the factual knowledge in the LLM, yet achieving strong performance gains and negligible forgetting.

\textbf{Continual Learning of Pre-Trained Image Encoders.}~~SPT~\cite{he2023sensitivity} estimates a mask of updates based on parameter sensitivity, performing low-rank or sparse updates. SPU~\cite{zhang2024overcoming} localizes updates to the first feed-forward layer of each transformer block, inspired by knowledge neuron theory~\cite{dai2021knowledge}. Our approach generalizes updates to all layers, selecting relevant parameters and maintaining gradient norms, combined with LoRA on selected attention heads for adaptivity and stability,  achieving SOTA performance on continual fewshot multimodal tasks.

\section{\texorpdfstring{\underline{Lo}w-\underline{R}ank Adaptation with \underline{S}tructured \underline{U}pdates}{Lg}}\label{sec:lorsu_method}


Few-shot continual learning is a highly practical and challenging scenario, where models must incrementally adapt to new tasks with limited supervision while retaining previously acquired knowledge. This setting closely mirrors real-world applications, such as interactive AI assistants and autonomous systems, where models receive a continuous stream of novel data but only sparse supervision per update.

To address the challenge of efficiently fine-tuning large-scale visual encoders and transformer-based models under the few-shot continual learning setting, without causing catastrophic forgetting (i.e., degradation in performance on previously learned tasks), we propose a novel parameter-efficient fine-tuning method called \emph{Low-Rank Adaptation with Structured Updates} (\textbf{LoRSU}).

LoRSU updates specific parameters within each transformer block in a resource-efficient manner, mitigating the risk of generic knowledge loss when fine-tuning for new tasks. Specifically, we selectively update a subset of parameters from the first linear layer in the MLP block of each transformer layer, as proposed in \cite{zhang2024overcoming}. While this approach reduces the fine-tuning burden, it may limit model flexibility as the remaining parameters in the transformer block remain fixed. To enhance flexibility, we further update the most informative attention heads based on the gradient of the task-specific loss.

More specifically, let a dataset $\cD_t = \{\bx_n, \by_n \}_{n=1}^{N_t}$ for the current task $t$ where $\bx_n$ is an image with text description $\by_n$. We define $\cL(\btheta; \cD_t) := \cL_t(\btheta)$ as the  loss used for training the model and $\btheta \in \mathbb{R}^d$ is the full set of model's parameters. The standard Multi-head Self-Attention Mechanism (MSA)~\cite{vaswani2017attention}, comprised of $H$ $D_h$-dimensional heads, is defined as the concatenation of multiple self-attention (SA) blocks where $\bq^{(i)} = W_q^{(i)} Z^{\top}, \bk^{(i)} = W_k^{(i)} Z^{\top}, \bv^{(i)} = W_v^{(i)} Z^{\top} \in \mathbb{R}^{D_h \times N}$, are the query, key and value matrices, which are used to compute the self-attention outputs as follows
\begin{align}
        A^{(i)} & = \text{softmax}( \bq^{(i)^\top} \bk^{(i)} / \sqrt{D_h} )  \in \mathbb{R}^{N \times N}, \\
        \text{SA}_i(Z) & = A^{(i)} \bv^{(i)^\top}  \in \mathbb{R}^{N \times D_h}, ~~i=1, \ldots, H.
\end{align}
$Z \in \mathbb{R}^{N \times D}$ is the input matrix of $N$ tokens of dimension $D$ and $ W_q^{(i)},  W_k^{(i)},$ and $ W_k^{(i)}$ are the query, key, and value matrices of learnable parameters for head $i$, respectively. The final MSA function is defined as 
$\text{MSA}(Z) = \text{Concat}\left[ SA_1(Z), \ldots, SA_H(Z)  \right] W_o \in \mathbb{R}^{N \times D},~W_o \in \mathbb{R}^{H D_h \times D}$.

Since we care to update the parameters of the heads that cause the largest changes in  $\cL_t(\btheta)$, we compute the gradient of the loss with respect to the parameters of each head and then we update only those heads with the largest cumulative contribution to the loss change. Since the matrices $W_q^{(i)},  W_k^{(i)},  W_v^{(i)}$ are all the parameters of head $i$, we can define an importance score for each head by adding the squared values of their corresponding gradients $G_q^{(i)} = \nabla_{W_q^{(i)}} \cL_t$, $G_k^{(i)} = \nabla_{W_k^{(i)}} \cL_t$, and $G_v^{(i)} = \nabla_{W_v^{(i)}} \cL_t$, as follows
\begin{equation}
    s_i = \sum_{m,l} \left( (G_q^{(i)}[m,l])^2 + (G_k^{(i)}[m,l])^2 + (G_v^{(i)}[m,l])^2 \right). \label{eq:s_i} 
\end{equation}
We provide a theoretical justification of \eqref{eq:s_i} in the next section. We update only the top-$k$ heads, based on their importance scores $\{s_1, \ldots, s_H \}$, $I \subset \{1, \ldots, H \}$, to be updated on the current task. Nevertheless, the number of parameters remain high due to the large weight matrices. Therefore, we parametrize the original weights using LoRA~\cite{hu2021lora} to further reduce the computational burden. The matrices $W_q^{(i)}, W_k^{(i)}, W_v^{(i)}, i \in I$ are now defined as
\begin{equation}
    W_{\alpha}^{(i)^{\prime}} = W_{\alpha}^{(i)} + A_{\alpha}^{(i)} B_{\alpha}^{(i)},~~~\alpha \in \{q, k, v \}. \label{eq:lora_attn_weights}
\end{equation}
Finally, to ensure that we only update $W_q^{(i)},  W_k^{(i)},  W_v^{(i)}, \forall i \in I$ we use a binary mask on the gradient vector with respect to all parameters of all attention heads. We keep the projection matrix $W_o$ frozen. We note that most modern implementations of transformer blocks concatenate the three attention weight matrices $W_q, W_k, W_v$ into one and thus we only need to apply LoRA once to this concatenated matrix.

Regarding the first linear layer in the MLP module, $W_{\text{fc1}} \in \mathbb{R}^{d \times D}$, we mask the gradients of $W_{\text{fc1}}$ so only the most important parameters for the current task to be updated, i.e.~we use the following biased gradient update.
\begin{equation}
\hat{\nabla}_{W_{\text{fc1}}} \cL_t = M_{\text{fc1}} \odot \nabla_{W_{\text{fc1}}} \cL_t,
\end{equation}
where  $M_{\text{fc1}} \in \{0, 1 \}^{d \times D}$ is a zero-one mask that is built by choosing a proportion of the largest squared values of $\nabla_{W_{\text{fc1}}} \cL_t$ in a similar manner as in~\cite{zhang2024overcoming} and $\odot$ is the Hadamard product.

\textbf{Theoretical justification.}~~The importance scores in \eqref{eq:s_i} can be derived from the following constrained (binary) optimization problem\footnote{For notational simplicity, we assume a single transformer block for this case.}
\begin{align}
 & \bp^*  =  \argmax_{\bp \in \{0, 1 \}^d} \frac{\norm{\bp \odot \nabla_W \cL(\btheta_0) }^2}{\norm{\nabla_W \cL(\btheta_0)}^2},  \label{eq:maxim} 
\\ & \text{s.t.}~~\bigcup_{\ell=1}^G I_{\ell} \subset \{1, 2, \ldots, d\},~I_i \cap I_j = \emptyset,~~\forall i\neq j, \nonumber \\
& \text{and}~~ C = \sum_{\ell=1}^G c_{\ell},~~c_{\ell} \leq |I_{\ell}|~~\forall \ell,~~~\norm{\bp}_0 \leq C, \nonumber
\end{align}
where $\btheta_0$ is the vector of the pretrained parameters before using $\cD_t$ for fine-tuning the model. The groups of parameters $I_i$ correspond to the parameters of a specific module (e.g. Self-Attention or MLP projector) we aim to learn, hence the constraint of mutually exclusiveness, $I_i \cap I_j = \emptyset$, between different pairs of parameter groups. Also note that we allowed to choose a subset $c_{\ell}$ of the parameters of a specific group $I_{\ell}$ which is the underneath mechanism of \ours choosing attention heads and parameters of fc1. The mask $\bp^*$ is chosen so that the gradient norm of the masked gradients is as large as possible under the sparsity constraints. We prove in appendix \ref{sec_appx:proof_mask} that the indices of the non-zero values of $\bp^*$ can be found using the importance scores in \eqref{eq:s_i} and the magnitudes of the gradients with respect to the fc1 parameters. 
\section{Experiments}\label{sec:experiments}
We conduct a series of experiments under three different few-shot continual learning (CL) settings (CL-5, CL-20, and CL-50 shots) to thoroughly investigate the performance of LoRSU based on ten VQA datasets. By adopting this paradigm, we aim to assess the adaptability and efficiency of LoRSU under constrained learning conditions, ensuring that it remains both computationally feasible and effective in improving downstream performance. Specifically, we seek to address the following key questions: 1) How does our method, LoRSU, compare to other fine-tuning and CL baselines that use the CLIP loss to update the image encoder? 2) Does updating the image encoder separately and then reintegrating it into the corresponding VLM enhance downstream VQA performance? 3) What is the effect of using the perplexity loss instead of the CLIP loss to update the image encoder? 4) What are the benefits of choosing a subset of attention heads to be fine-tuned using LoRSU? and 5) What are the computational benefits of LoRSU?

\subsection{Datasets}\label{sec:datasets}
We evaluate the performance of \ours on ten visual question answering (VQA) datasets falling in two broad categories: regular VQA datasets and classification datasets converted to VQA datasets.

\textbf{Regular VQA datasets.}~~We consider four standard VQA datasets used for benchmarking VLMs' performance~\cite{duan2024vlmevalkit}: 
\emph{VSR}~\cite{Liu2022VisualSR}, the Visual Spatial Reasoning corpus consists of caption-image pairs labeled as True or False, where each caption describes the spatial relation between two objects in the image. VLMs evaluate whether the caption accurately reflects the image.
\emph{HM}~\cite{kiela2020hateful}, the Hateful Memes dataset designed to detect multimodal hateful memes.
\emph{MMVP}~\cite{tong2024eyes}, the Multimodal Visual Patterns
dataset is a challenging benchmark which has been built on images that CLIP perceives as similar despite their clear visual differences.
\emph{VisOnly}~\cite{kamoi2024visonlyqa}, a novel dataset created to directly assess the visual perception abilities of VLMs in answering questions about geometric and numerical details in scientific figures. This dataset allows us to assess fine-grained visual perception in VLMs independently of other abilities, such as reasoning, making it the most challenging among the previously mentioned datasets.

\textbf{Classification-to-VQA datasets.}~~We convert four popular multi-class classification datasets into multiple-choice VQA problems, where each question has five choices, and the VLM is tasked with selecting the correct answer. These datasets are introduced as examples of scenarios where visual domain shifts are encountered, allowing us to examine the utility of updating the image encoder; a critical consideration often overlooked in many standard VQA datasets.
The datasets include:
\emph{GTS}~\cite{stallkamp2012man}, the German Traffic Sign dataset, which \citet{zhang2024overcoming} considered as an out-of-distribution dataset for CLIP pretraining; \emph{CAn}~\cite{wang2024clips}, a recent dataset created to test CLIP's robustness with animal images containing realistic spurious features such as unexpected backgrounds; \emph{AIR}~\cite{maji13fine-grained}, a fine-grained aircraft classification dataset; \emph{ESAT}~\cite{helber2019eurosat}, a dataset of satellite images used for land cover classification.

\textbf{TSI \& DALLE.}~~In addition to these existing datasets, we introduce two novel VQA datasets: TSI and DALLE, both designed to explore the effects of domain shift.
The \emph{TSI}~\cite{das2019toyota} dataset was preprocessed as a classification dataset, where the goal is to recognize the activity depicted in each image. Frames were extracted from videos to create a training set of approximately 10K images and a test set of approximately 5K images, encompassing 27 distinct activity classes. The \emph{DALLE} dataset, constructed by querying the OpenAI's model DALL·E 2, includes representative images generated from 22 activity classes appearing in TSI. For each activity, we generated 30 images, resulting in a total of 660 images designated exclusively for evaluation purposes. 

We follow the common practice in few-shot continual learning~\cite{panos2023first} to construct the sequences. We divide each dataset into 5 sets of disjoint classes/categories and consider 5/20/50 shot settings where only 5/20/50 images per class in the current set are used for fine-tuning the model. More details on how we split each of these datasets for the CL settings are provided in appendix~\ref{sec_appx:datasets}.

\begin{table*}[!ht]
\caption{Performance comparison of LoRSU with the CLIP loss against baselines fine-tuning the image encoder using the same loss. We report the \emph{Target Improvement} (TI) and \emph{Control Change} (CC) accuracies across three different continual learning (CL) settings. Greener shades indicate higher positive values, while redder shades signify lower negative values. The highest accuracies across methods for each dataset are underlined.}
 \label{table:clip_baselines_summary}
\vskip 0.15in
\begin{center}
\begin{small}
\begingroup
\setlength{\tabcolsep}{1.8pt}
\begin{tabular}{l c c c c c c c c c c c c c c c}
\toprule
 & & \multicolumn{14}{c}{\textbf{FT Method}}  \\
\cmidrule(lr){3-16}
\multirow{2}{*}{\textbf{Setting}} & \multirow{2}{*}{\textbf{FT Dataset}}  &  \multicolumn{2}{c}{\textbf{LN}} &  \multicolumn{2}{c}{\textbf{F-FT}} &  \multicolumn{2}{c}{\textbf{F-EWC}} &  \multicolumn{2}{c}{\textbf{LoRA}} &  \multicolumn{2}{c}{\textbf{AdaLoRA}} &  \multicolumn{2}{c}{\textbf{SPU}} &  \multicolumn{2}{c}{\textbf{LoRSU}} \\
\cmidrule(lr){3-4} \cmidrule(lr){5-6} \cmidrule(lr){7-8} \cmidrule(lr){9-10} \cmidrule(lr){11-12} \cmidrule(lr){13-14} \cmidrule(lr){15-16} & & \textbf{TI ($\uparrow)$} & \textbf{CC ($\uparrow)$} & \textbf{TI ($\uparrow)$} & \textbf{CC ($\uparrow)$} & \textbf{TI ($\uparrow)$} & \textbf{CC ($\uparrow)$} & \textbf{TI ($\uparrow)$} & \textbf{CC ($\uparrow)$} & \textbf{TI ($\uparrow)$} & \textbf{CC ($\uparrow)$} & \textbf{TI ($\uparrow)$} & \textbf{CC ($\uparrow)$} & \textbf{TI ($\uparrow)$} & \textbf{CC ($\uparrow)$} \\
\midrule
\multirow{5}{*}{\textbf{CL-5}} & \textbf{GTS} & \colorbox{green!30}{\phantom{-1}3.5} & \colorbox{red!12}{\phantom{1}-1.5} & \colorbox{green!30}{\phantom{-1}3.7} & \colorbox{red!70}{\phantom{1}-6.5} & \colorbox{green!}{\phantom{-1}5.0} & \colorbox{red!70}{-11.5} & \colorbox{green!15}{\phantom{-1}0.7} & \colorbox{red!70}{\phantom{1}-4.8} & \colorbox{red!25}{\phantom{1}-0.9} & \colorbox{red!70}{\phantom{1}-4.9} & \colorbox{green!}{\phantom{-1}5.4} & \colorbox{red!6}{\phantom{1}\underline{-0.6}} & \colorbox{green!}{\phantom{-1}\underline{6.4}} & \colorbox{red!12}{\phantom{1}-0.7} \\
& \textbf{TSI} & \colorbox{green!15}{\phantom{-1}0.8} & \colorbox{green!25}{\phantom{-1}0.0} & \colorbox{green!}{\phantom{-1}7.4} & \colorbox{red!12}{\phantom{1}-1.1} & \colorbox{green!}{\phantom{-1}\underline{8.5}} & \colorbox{red!12}{\phantom{1}-1.0} & \colorbox{red!25}{\phantom{1}-0.1} & \colorbox{red!50}{\phantom{1}-2.8} & \colorbox{green!15}{\phantom{-1}1.1} & \colorbox{green!}{\phantom{-1}\underline{0.2}} & \colorbox{green!15}{\phantom{-1}0.9} & \colorbox{green!50}{\phantom{-1}0.1} & \colorbox{green!30}{\phantom{-1}3.2} & \colorbox{green!50}{\phantom{-1}0.1} \\
& \textbf{CAn} & \colorbox{red!50}{\phantom{1}-2.4} & \colorbox{red!6}{\phantom{1}-0.2} & \colorbox{red!50}{\phantom{1}-2.4} & \colorbox{red!25}{\phantom{1}-2.2} & \colorbox{red!70}{-16.7} & \colorbox{red!70}{\phantom{1}-9.4} & \colorbox{red!25}{\phantom{1}-1.3} & \colorbox{red!70}{\phantom{1}-4.6} & \colorbox{red!25}{\phantom{1}-1.0} & \colorbox{red!6}{\phantom{1}-0.1} & \colorbox{red!25}{\phantom{1}-0.4} & \colorbox{green!50}{\phantom{-1}0.1} & \colorbox{green!15}{\phantom{-1}\underline{0.3}} & \colorbox{green!}{\phantom{-1}\underline{0.3}} \\
& \textbf{AIR} & \colorbox{green!15}{\phantom{-1}0.3} & \colorbox{red!12}{\phantom{1}-1.6} & \colorbox{green!15}{\phantom{-1}2.0} & \colorbox{red!50}{\phantom{1}-2.7} & \colorbox{green!30}{\phantom{-1}2.9} & \colorbox{red!50}{\phantom{1}-2.8} & \colorbox{green!15}{\phantom{-1}1.3} & \colorbox{red!50}{\phantom{1}-3.7} & \colorbox{green!15}{\phantom{-1}0.4} & \colorbox{green!25}{\phantom{-1}0.0} & \colorbox{green!30}{\phantom{-1}3.1} & \colorbox{green!50}{\phantom{-1}0.1} & \colorbox{green!30}{\phantom{-1}\underline{4.8}} & \colorbox{green!}{\phantom{-1}\underline{0.4}} \\
& \textbf{ESAT} & \colorbox{green!30}{\phantom{-1}4.2} & \colorbox{green!}{\phantom{-1}\underline{0.6}} & \colorbox{red!70}{-10.3} & \colorbox{red!12}{\phantom{1}-1.4} & \colorbox{red!70}{\phantom{1}-8.4} & \colorbox{red!25}{\phantom{1}-2.1} & \colorbox{red!25}{\phantom{1}-1.6} & \colorbox{red!12}{\phantom{1}-0.7} & \colorbox{green!15}{\phantom{-1}1.9} & \colorbox{green!25}{\phantom{-1}0.1} & \colorbox{green!30}{\phantom{-1}4.5} & \colorbox{green!50}{\phantom{-1}0.1} & \colorbox{green!}{\phantom{-1}\underline{6.8}} & \colorbox{green!}{\phantom{-1}0.2} \\
\midrule
\multirow{5}{*}{\textbf{CL-20}} & \textbf{GTS} & \colorbox{green!}{\phantom{-1}5.2} & \colorbox{red!70}{\phantom{1}-5.9} & \colorbox{green!30}{\phantom{-1}4.6} & \colorbox{red!70}{\phantom{1}-7.3} & \colorbox{green!}{\phantom{-1}6.7} & \colorbox{red!70}{-15.6} & \colorbox{green!30}{\phantom{-1}2.5} & \colorbox{red!70}{-10.5} & \colorbox{green!15}{\phantom{-1}0.2} & \colorbox{red!25}{\phantom{1}-2.2} & \colorbox{green!}{\phantom{-1}7.9} & \colorbox{red!12}{\phantom{1}-1.3} & \colorbox{green!}{\phantom{-1}\underline{8.6}} & \colorbox{red!12}{\phantom{1}\underline{-1.0}} \\
& \textbf{TSI} & \colorbox{green!}{\phantom{-1}5.1} & \colorbox{red!25}{\phantom{1}-1.9} & \colorbox{green!}{\phantom{-}15.3} & \colorbox{red!50}{\phantom{1}-3.4} & \colorbox{green!}{\phantom{-}\underline{16.0}} & \colorbox{red!70}{-32.5} & \colorbox{green!}{\phantom{-1}8.5} & \colorbox{red!70}{\phantom{1}-4.4} & \colorbox{green!15}{\phantom{-1}1.3} & \colorbox{red!70}{\phantom{1}-9.6} & \colorbox{green!}{\phantom{-1}7.8} & \colorbox{red!6}{\phantom{1}-0.3} & \colorbox{green!}{\phantom{-}10.6} & \colorbox{red!6}{\phantom{1}\underline{-0.1}} \\
& \textbf{CAn} & \colorbox{red!50}{\phantom{1}-2.4} & \colorbox{red!6}{\phantom{1}-0.4} & \colorbox{green!15}{\phantom{-1}0.3} & \colorbox{red!50}{\phantom{1}-2.9} & \colorbox{green!15}{\phantom{-1}0.1} & \colorbox{red!70}{\phantom{1}-5.1} & \colorbox{red!50}{\phantom{1}-2.3} & \colorbox{red!70}{\phantom{1}-5.4} & \colorbox{red!70}{\phantom{1}-3.5} & \colorbox{red!25}{\phantom{1}-2.5} & \colorbox{green!15}{\phantom{-1}0.1} & \colorbox{green!}{\phantom{-1}0.5} & \colorbox{green!15}{\phantom{-1}\underline{1.1}} & \colorbox{green!}{\phantom{-1}\underline{0.3}} \\
& \textbf{AIR} & \colorbox{red!25}{\phantom{1}-0.2} & \colorbox{red!50}{\phantom{1}-3.0} & \colorbox{green!}{\phantom{-1}9.3} & \colorbox{red!25}{\phantom{1}-1.8} & \colorbox{green!}{\phantom{-}\underline{10.2}} & \colorbox{red!25}{\phantom{1}-2.0} & \colorbox{green!}{\phantom{-1}5.3} & \colorbox{red!50}{\phantom{1}-2.7} & \colorbox{green!30}{\phantom{-1}2.7} & \colorbox{red!12}{\phantom{1}-0.7} & \colorbox{green!30}{\phantom{-1}3.0} & \colorbox{red!6}{\phantom{1}\underline{-0.2}} & \colorbox{green!}{\phantom{-1}5.9} & \colorbox{red!6}{\phantom{1}-0.5} \\
& \textbf{ESAT} & \colorbox{green!15}{\phantom{-1}0.9} & \colorbox{red!6}{\phantom{1}-0.1} & \colorbox{red!70}{-24.9} & \colorbox{red!25}{\phantom{1}-1.7} & \colorbox{red!70}{-22.0} & \colorbox{red!50}{\phantom{1}-3.8} & \colorbox{red!70}{-11.5} & \colorbox{red!6}{\phantom{1}-0.5} & \colorbox{red!70}{\phantom{1}-6.8} & \colorbox{red!50}{\phantom{1}-2.7} & \colorbox{green!}{\phantom{-1}5.4} & \colorbox{green!}{\phantom{-1}\underline{0.3}} & \colorbox{green!}{\phantom{-1}\underline{6.6}} & \colorbox{green!}{\phantom{-1}0.2} \\
\midrule
\multirow{5}{*}{\textbf{CL-50}} & \textbf{GTS} & \colorbox{green!30}{\phantom{-1}4.8} & \colorbox{red!70}{\phantom{1}-6.5} & \colorbox{green!30}{\phantom{-1}3.4} & \colorbox{red!70}{\phantom{1}-9.8} & \colorbox{green!}{\phantom{-1}5.3} & \colorbox{red!70}{-12.9} & \colorbox{green!30}{\phantom{-1}3.1} & \colorbox{red!70}{-11.1} & \colorbox{green!15}{\phantom{-1}1.0} & \colorbox{red!50}{\phantom{1}-3.3} & \colorbox{green!}{\phantom{-1}7.7} & \colorbox{red!12}{\phantom{1}-1.5} & \colorbox{green!}{\phantom{-1}\underline{9.7}} & \colorbox{red!12}{\phantom{1}\underline{-1.3}} \\
& \textbf{TSI} & \colorbox{green!}{\phantom{-1}7.0} & \colorbox{red!50}{\phantom{1}-3.0} & \colorbox{green!}{\phantom{-}17.2} & \colorbox{red!70}{\phantom{1}-4.6} & \colorbox{green!}{\phantom{-}\underline{22.4}} & \colorbox{red!70}{-13.4} & \colorbox{green!}{\phantom{-}18.2} & \colorbox{red!70}{\phantom{1}-6.3} & \colorbox{green!}{\phantom{-1}7.9} & \colorbox{red!25}{\phantom{1}-1.9} & \colorbox{green!}{\phantom{-}12.2} & \colorbox{red!6}{\phantom{1}-0.5} & \colorbox{green!}{\phantom{-}19.1} & \colorbox{red!6}{\phantom{1}\underline{-0.3}} \\
& \textbf{CAn} & \colorbox{red!70}{\phantom{1}-5.7} & \colorbox{red!50}{\phantom{1}-3.3} & \colorbox{red!25}{\phantom{1}-1.0} & \colorbox{red!70}{\phantom{1}-4.9} & \colorbox{green!15}{\phantom{-1}0.6} & \colorbox{red!70}{\phantom{1}-9.7} & \colorbox{red!25}{\phantom{1}-0.4} & \colorbox{red!70}{\phantom{1}-4.4} & \colorbox{red!25}{\phantom{1}-1.8} & \colorbox{red!12}{\phantom{1}-0.8} & \colorbox{green!15}{\phantom{-1}0.6} & \colorbox{red!6}{\phantom{1}\underline{-0.3}} & \colorbox{green!15}{\phantom{-1}\underline{1.3}} & \colorbox{red!6}{\phantom{1}-0.5} \\
& \textbf{AIR} & \colorbox{green!15}{\phantom{-1}1.8} & \colorbox{red!50}{\phantom{1}-3.9} & \colorbox{green!}{\phantom{-}10.0} & \colorbox{red!50}{\phantom{1}-3.1} & \colorbox{green!}{\phantom{-}\underline{10.9}} & \colorbox{red!50}{\phantom{1}-3.3} & \colorbox{green!}{\phantom{-1}7.8} & \colorbox{red!50}{\phantom{1}-3.8} & \colorbox{green!30}{\phantom{-1}4.6} & \colorbox{red!12}{\phantom{1}-0.9} & \colorbox{green!}{\phantom{-1}6.2} & \colorbox{red!12}{\phantom{1}\underline{-0.6}} & \colorbox{green!}{\phantom{-1}8.2} & \colorbox{red!12}{\phantom{1}-0.7} \\
& \textbf{ESAT} & \colorbox{green!30}{\phantom{-1}4.6} & \colorbox{green!50}{\phantom{-1}0.1} & \colorbox{red!70}{-41.4} & \colorbox{red!50}{\phantom{1}-3.3} & \colorbox{red!70}{-38.1} & \colorbox{red!25}{\phantom{1}-2.0} & \colorbox{red!70}{-14.5} & \colorbox{red!50}{\phantom{1}-3.6} & \colorbox{red!70}{-17.3} & \colorbox{red!25}{\phantom{1}-2.4} & \colorbox{green!}{\phantom{-1}5.8} & \colorbox{green!50}{\phantom{-1}0.1} & \colorbox{green!}{\phantom{-1}\underline{7.0}} & \colorbox{green!50}{\phantom{-1}\underline{0.2}} \\
\bottomrule
\end{tabular}
\endgroup
\end{small}
\end{center}
\vskip -0.1in
\end{table*}

\subsection{Experimental Setting}\label{sec:experimental_settings}

\textbf{Metrics.}~~
While standard metrics in the CL literature exist to evaluate general performance~\cite{lopez2017gradient,chaudhry2018riemannian}, VLMs exhibit generic knowledge across various domains beyond the one being adapted, making it crucial to evaluate how adaptation impacts their overall performance. These metrics do not measure the change in performance relative to the model's initial state prior to the learning process.
To address this, we use the zero-shot accuracy of each VQA dataset as the benchmark baseline and report the change in accuracy on the test split of the target dataset so positive values indicate an improvement in accuracy. This approach enables us to quantify the model's ability to accumulate knowledge, using the pretrained model as the reference point; we name this metric as \emph{Target Improvement} (TI) accuracy. We also calculate the average accuracy change on the test splits of the remaining datasets, when fine-tuning on a specific dataset, to estimate average forgetting of generic knowledge or possible positive backward transfer~\cite{de2021continual}; we call this metric \emph{Control Change} (CC) accuracy where `control' refers to the control datasets we use to calculate the average accuracy change. TI and CC are computed based on the fine-tuned VLM after the last session of CL. We also consider standard CL performance metrics such as \emph{Average Accuracy} (ACC) and \emph{Backward Transfer} (BWT)~\cite{lopez2017gradient} to examine how accuracy and forgetting evolves through continuous adaptation. Notice that these metrics, in contrast to TI and CC, focus on the accuracy and forgetting during continual adaptation and they do not take into account the performance of the fine-tuned model on other datasets.

\textbf{Implementation details.}~~Please see Appendix~\ref{sec_appx:implementation_details}.

\textbf{Models.}~~For our experiments, we consider the popular Vision Language Model LLaVA-v1.5~\cite{liu2024visual} that leverages a frozen CLIP image encoder. Specifically, LLaVA utilizes a frozen OpenAI-CLIP-L-14~\cite{radford2021learning} with a LLM (Vicuna-7b~\cite{chiang2023vicuna}). The two modules are connected through a two-layer MLP projector that aligns image and text features. The LLM and the MLP projector are optimized during the visual instruction tuning while CLIP remains frozen. LLaVA concatenates adjacent tokens from CLIP-L-14 and processes them with an MLP projector as input to LLama-2 (7B-chat)~\cite{touvron2023llama}; the MLP projector and the language model are optimized while the image encoder remains frozen. 

\textbf{Baselines.}~~We compare LoRSU to the following methods that also use the CLIP loss to fine-tune the image encoder:
\begin{itemize}[noitemsep,topsep=1pt,parsep=1pt,partopsep=1pt,leftmargin=*]
    \item \emph{LN}~\cite{perez2018film,panos2023first} is used for both few-shot and CL. Only the  image encoder LayerNorm modules' parameters are optimized.
    \item \emph{F-FT} is the standard fine-tuning technique where all image encoder parameters undergo gradient updates.
    \item \emph{F-EWC} fine-tunes all the image encoder parameters with EWC regularization~\cite{kirkpatrick2017overcoming}.
    \item \emph{LoRA}~\cite{hu2021lora} a popular PEFT method which parameterizes incremental updates by two low-dimensional matrices and only fine-tunes them.
    \item \emph{AdaLoRA}~\cite{zhang2023adalora} dynamically adjusts the low-rank update budget allocation during training.
    \item \emph{SPU}~\cite{zhang2024overcoming} is a PEFT baseline, specifically designed to tackle catastrophic forgetting in CL scenarios, that utilizes structured sparsity based on gradient information to fine-tune the most significant parameters of the fc1 module in the transformer block. 
\end{itemize}

\subsection{CLIP-based Updates}

We evaluate the performance of the Vision-Language Model (VLM) when only the image encoder is fine-tuned using the CLIP loss in a CL setting. This experiment compares six strong CLIP-based baselines with our proposed method, \ours. Table~\ref{table:clip_baselines_summary} reports the average accuracies of TI/CC over three runs; detailed results can be found in appendix~\ref{sec_appx:detailed_res}. We observe that \ours consistently achieves superior TI scores across datasets and CL settings, underscoring its ability to enhance task-specific performance effectively. Furthermore, \ours maintains CC accuracies that take consistently small negative or even positive values, highlighting its capacity to preserve or slightly improve performance on control datasets while fine-tuning on target datasets. Even in datasets where other methods struggle (e.g., CAn, ESAT), LoRSU often performs better, maintaining positive or close-to-neutral TI and CC scores. For instance, In ESAT (CL-50) containing challenging satellite images, LoRSU achieves the highest TI (7.0) with a positive CC (0.2), outperforming SPU (TI=5.8, CC=0.1) and all other methods. 

\begin{table}
\caption{\emph{Average accuracy} (ACC) ($\uparrow$) and \emph{backward transfer} (BWT) ($\uparrow$) scores (\%). For reference, the ACC of the pretrained model on GTS and ESAT is $75.4$ and $76.4$, respectively, while BWT is zero for all cases. The highest scores across methods are underlined.}
\label{table:bwt_metrics_clip_reduced}
\vskip 0.15in
\begin{center}
\begin{small}
\begingroup
\setlength{\tabcolsep}{2.7pt}
\begin{tabular}{l c c c c c c c}
\toprule
\multirow{2}{*}{\textbf{Setting}} & \multirow{2}{*}{\textbf{FT Dataset}}  &  \multicolumn{2}{c}{\textbf{LoRA}} & \multicolumn{2}{c}{\textbf{SPU}} &  \multicolumn{2}{c}{\textbf{LoRSU}} \\
\cmidrule(lr){3-4} \cmidrule(lr){5-6} \cmidrule(lr){7-8} & & \textbf{ACC} & \textbf{BWT} & \textbf{ACC} & \textbf{BWT} & \textbf{ACC} & \textbf{BWT } \\
\midrule
\multirow{2}{*}{\textbf{CL-5}} & \textbf{GTS} & 79.2 & -7.1 & 80.8 & \underline{0.5} & \underline{81.1} & 0.4 \\
& \textbf{ESAT} & 73.8 & -3.4 & 79.8 & 1.5 & \underline{82.2} & \underline{2.0} \\
\midrule
\multirow{2}{*}{\textbf{CL-20}} & \textbf{GTS} & 77.2 & -9.1 & 82.8 & -0.6 & \underline{83.5} & \underline{-0.4} \\
& \textbf{ESAT} & 64.1 & -18.3 & 82.0 & \underline{2.0} & \underline{82.7} & 0.1 \\
\midrule
\multirow{2}{*}{\textbf{CL-50}} & \textbf{GTS} & 79.3 & -10.3 & 83.8 & -0.7 & \underline{84.7} & \underline{-0.5} \\
& \textbf{ESAT} & 61.4 & -27.8 & 81.2 & -2.4 & \underline{82.1} & \underline{-0.8} \\
\bottomrule
\end{tabular}
\endgroup
\end{small}
\end{center}
\vskip -0.1in
\end{table}

\begin{table*}[!ht]
\caption{Performance comparison between LoRSU using the CLIP loss (\emph{LoRSU}) or the perplexity loss (LoRSU-Ppl) and other baselines that fine-tune only the vision encoder (\emph{LoRA, LoRA-Ppl}), only the LLM (\emph{LoRA-L}), or both of them (\emph{LoRA-F}). We report the \emph{Target Improvement} (TI) and \emph{Control Change} (CC) for each CL setting. $\dagger$ and $\ddagger$ denote classification-to-VQA and regular VQA datasets, respectively. The highest accuracies across methods for each dataset are underlined.}
 \label{table:ppl_vs_clip_summary}
\vskip 0.15in
\begin{center}
\begin{small}
\begingroup
\setlength{\tabcolsep}{3.9pt}
\begin{tabular}{l c c c c c c c c c c c c c}
\toprule
 \multirow{3}{*}{\textbf{Setting}} &  \multirow{3}{*}{\textbf{FT Dataset}} & \multicolumn{12}{c}{\textbf{FT Method}}  \\
\cmidrule(lr){3-14}
 &  & \multicolumn{2}{c}{\textbf{LoRA-L}} & \multicolumn{2}{c}{\textbf{LoRA}} &  \multicolumn{2}{c}{\textbf{LoRSU}} & \multicolumn{2}{c}{\textbf{LoRA-Ppl}} & \multicolumn{2}{c}{\textbf{LoRA-F}} & \multicolumn{2}{c}{\textbf{LoRSU-Ppl}} \\
\cmidrule(lr){3-4} \cmidrule(lr){5-6} \cmidrule(lr){7-8} \cmidrule(lr){9-10} \cmidrule(lr){11-12} \cmidrule(lr){13-14} & & \textbf{TI ($\uparrow)$} & \textbf{CC ($\uparrow)$} & \textbf{TI ($\uparrow)$} & \textbf{CC ($\uparrow)$} & \textbf{TI ($\uparrow)$} & \textbf{CC ($\uparrow)$} & \textbf{TI ($\uparrow)$} & \textbf{CC ($\uparrow)$} & \textbf{TI ($\uparrow)$} & \textbf{CC ($\uparrow)$} & \textbf{TI ($\uparrow)$} & \textbf{CC ($\uparrow)$} \\
\midrule
\multirow{8}{*}{\textbf{CL-5}} & $\textbf{GTS}^{\dagger}$ & \colorbox{red!70}{\phantom{1}-4.1} & \colorbox{red!6}{\phantom{1}\underline{-0.2}} & \colorbox{green!15}{\phantom{-1}0.7} & \colorbox{red!70}{\phantom{1}-4.8} & \colorbox{green!}{\phantom{-1}\underline{6.4}} & \colorbox{red!12}{\phantom{1}-0.7} & \colorbox{red!70}{\phantom{1}-7.5} & \colorbox{red!50}{\phantom{1}-3.0} & \colorbox{red!50}{\phantom{1}-2.7} & \colorbox{red!25}{\phantom{1}-1.8} & \colorbox{green!15}{\phantom{-1}1.6} & \colorbox{red!12}{\phantom{1}-1.0} \\
& $\textbf{TSI}^{\dagger}$ & \colorbox{green!}{\phantom{-1}6.0} & \colorbox{red!6}{\phantom{1}-0.1} & \colorbox{red!25}{\phantom{1}-0.1} & \colorbox{red!50}{\phantom{1}-2.8} & \colorbox{green!30}{\phantom{-1}3.2} & \colorbox{green!50}{\phantom{-1}0.1} & \colorbox{green!}{\phantom{-}10.9} & \colorbox{red!25}{\phantom{1}-2.4} & \colorbox{red!70}{\phantom{1}-8.0} & \colorbox{red!25}{\phantom{1}-2.4} & \colorbox{green!}{\phantom{-}\underline{13.1}} & \colorbox{green!}{\phantom{-1}\underline{1.5}} \\
& $\textbf{CAn}^{\dagger}$ & \colorbox{red!70}{\phantom{1}-3.3} & \colorbox{red!6}{\phantom{1}-0.2} & \colorbox{red!25}{\phantom{1}-1.3} & \colorbox{red!70}{\phantom{1}-4.6} & \colorbox{green!15}{\phantom{-1}\underline{0.3}} & \colorbox{green!}{\phantom{-1}\underline{0.3}} & \colorbox{red!70}{\phantom{1}-3.5} & \colorbox{red!70}{\phantom{1}-5.5} & \colorbox{red!70}{\phantom{1}-4.1} & \colorbox{red!12}{\phantom{1}-1.6} & \colorbox{green!15}{\phantom{-1}0.2} & \colorbox{red!6}{\phantom{1}-0.2} \\
& $\textbf{AIR}^{\dagger}$ & \colorbox{red!25}{\phantom{1}-1.7} & \colorbox{green!}{\phantom{-1}0.3} & \colorbox{green!15}{\phantom{-1}1.3} & \colorbox{red!50}{\phantom{1}-3.7} & \colorbox{green!30}{\phantom{-1}4.8} & \colorbox{green!}{\phantom{-1}\underline{0.4}} & \colorbox{red!25}{\phantom{1}-0.7} & \colorbox{red!12}{\phantom{1}-1.5} & \colorbox{green!}{\phantom{-1}\underline{9.6}} & \colorbox{red!25}{\phantom{1}-1.9} & \colorbox{green!}{\phantom{-1}5.8} & \colorbox{red!6}{\phantom{1}-0.2} \\
& $\textbf{ESAT}^{\dagger}$ & \colorbox{red!25}{\phantom{1}-0.2} & \colorbox{red!6}{\phantom{1}-0.1} & \colorbox{red!25}{\phantom{1}-1.6} & \colorbox{red!12}{\phantom{1}-0.7} & \colorbox{green!}{\phantom{-1}\underline{6.8}} & \colorbox{green!}{\phantom{-1}0.2} & \colorbox{red!25}{\phantom{1}-0.6} & \colorbox{green!}{\phantom{-1}\underline{0.4}} & \colorbox{green!}{\phantom{-1}5.4} & \colorbox{red!6}{\phantom{1}-0.5} & \colorbox{green!30}{\phantom{-1}3.7} & \colorbox{green!25}{\phantom{-1}0.1} \\
& $\textbf{VSR}^{\ddagger}$ & \colorbox{green!}{\phantom{-}16.8} & \colorbox{red!6}{\phantom{1}-0.6} & \colorbox{green!15}{\phantom{-1}0.5} & \colorbox{red!50}{\phantom{1}-4.0} & \colorbox{green!15}{\phantom{-1}0.4} & \colorbox{green!}{\phantom{-1}\underline{0.2}} & \colorbox{green!}{\phantom{-}10.2} & \colorbox{red!70}{-12.5} & \colorbox{green!}{\phantom{-}\underline{18.0}} & \colorbox{red!70}{-10.6} & \colorbox{green!}{\phantom{-}10.5} & \colorbox{red!12}{\phantom{1}-1.2} \\
& $\textbf{HM}^{\ddagger}$ & \colorbox{green!}{\phantom{-1}\underline{7.4}} & \colorbox{red!50}{\phantom{1}-2.7} & \colorbox{red!25}{\phantom{1}-0.4} & \colorbox{red!70}{\phantom{1}-6.8} & \colorbox{green!15}{\phantom{-1}0.6} & \colorbox{green!}{\phantom{-1}\underline{0.4}} & \colorbox{red!25}{\phantom{1}-1.2} & \colorbox{red!12}{\phantom{1}-1.2} & \colorbox{green!}{\phantom{-1}6.0} & \colorbox{red!70}{\phantom{1}-4.5} & \colorbox{red!25}{\phantom{1}-0.8} & \colorbox{green!}{\phantom{-1}0.2} \\
& $\textbf{VisOnly}^{\ddagger}$ & \colorbox{red!25}{\phantom{1}-0.4} & \colorbox{red!6}{\phantom{1}-0.1} & \colorbox{red!25}{\phantom{1}-1.1} & \colorbox{red!70}{\phantom{1}-4.5} & \colorbox{green!15}{\phantom{-1}0.9} & \colorbox{green!50}{\phantom{-1}0.1} & \colorbox{green!15}{\phantom{-1}0.3} & \colorbox{red!6}{\phantom{1}-0.3} & \colorbox{green!15}{\phantom{-1}0.2} & \colorbox{red!6}{\phantom{1}-0.4} & \colorbox{green!30}{\phantom{-1}\underline{2.7}} & \colorbox{green!}{\phantom{-1}\underline{0.7}} \\
\midrule
\multirow{8}{*}{\textbf{CL-20}} & $\textbf{GTS}^{\dagger}$ & \colorbox{red!25}{\phantom{1}-1.4} & \colorbox{green!25}{\phantom{-1}\underline{0.1}} & \colorbox{green!30}{\phantom{-1}2.5} & \colorbox{red!70}{-10.5} & \colorbox{green!}{\phantom{-1}\underline{8.6}} & \colorbox{red!12}{\phantom{1}-1.0} & \colorbox{red!25}{\phantom{1}-0.5} & \colorbox{red!70}{\phantom{1}-6.4} & \colorbox{red!25}{\phantom{1}-1.4} & \colorbox{red!12}{\phantom{1}-0.8} & \colorbox{green!30}{\phantom{-1}3.9} & \colorbox{red!12}{\phantom{1}-0.7} \\
& $\textbf{TSI}^{\dagger}$ & \colorbox{green!}{\phantom{-1}5.9} & \colorbox{green!25}{\phantom{-1}\underline{0.0}} & \colorbox{green!}{\phantom{-1}8.5} & \colorbox{red!70}{\phantom{1}-4.4} & \colorbox{green!}{\phantom{-}10.6} & \colorbox{red!6}{\phantom{1}-0.1} & \colorbox{green!}{\phantom{-1}6.5} & \colorbox{red!70}{-11.6} & \colorbox{green!30}{\phantom{-1}2.9} & \colorbox{red!50}{\phantom{1}-3.1} & \colorbox{green!}{\phantom{-}\underline{13.9}} & \colorbox{red!6}{\phantom{1}-0.6} \\
& $\textbf{CAn}^{\dagger}$ & \colorbox{red!25}{\phantom{1}-1.9} & \colorbox{red!6}{\phantom{1}-0.6} & \colorbox{red!50}{\phantom{1}-2.3} & \colorbox{red!70}{\phantom{1}-5.4} & \colorbox{green!15}{\phantom{-1}\underline{1.1}} & \colorbox{green!}{\phantom{-1}\underline{0.3}} & \colorbox{red!70}{\phantom{1}-3.7} & \colorbox{red!70}{\phantom{1}-8.8} & \colorbox{red!50}{\phantom{1}-2.1} & \colorbox{red!25}{\phantom{1}-1.7} & \colorbox{green!15}{\phantom{-1}0.5} & \colorbox{red!12}{\phantom{1}-1.2} \\
& $\textbf{AIR}^{\dagger}$ & \colorbox{green!30}{\phantom{-1}3.7} & \colorbox{green!}{\phantom{-1}\underline{0.3}} & \colorbox{green!}{\phantom{-1}5.3} & \colorbox{red!50}{\phantom{1}-2.7} & \colorbox{green!}{\phantom{-1}5.9} & \colorbox{red!6}{\phantom{1}-0.5} & \colorbox{green!30}{\phantom{-1}4.8} & \colorbox{red!50}{\phantom{1}-3.5} & \colorbox{green!}{\phantom{-}\underline{16.3}} & \colorbox{red!6}{\phantom{1}-0.3} & \colorbox{green!}{\phantom{-1}6.0} & \colorbox{red!6}{\phantom{1}-0.3} \\
& $\textbf{ESAT}^{\dagger}$ & \colorbox{green!15}{\phantom{-1}0.7} & \colorbox{green!}{\phantom{-1}\underline{0.4}} & \colorbox{red!70}{-11.5} & \colorbox{red!6}{\phantom{1}-0.5} & \colorbox{green!}{\phantom{-1}\underline{6.6}} & \colorbox{green!}{\phantom{-1}0.2} & \colorbox{red!25}{\phantom{1}-1.2} & \colorbox{red!6}{\phantom{1}-0.1} & \colorbox{red!70}{\phantom{1}-4.6} & \colorbox{red!6}{\phantom{1}-0.0} & \colorbox{green!30}{\phantom{-1}2.9} & \colorbox{red!6}{\phantom{1}-0.1} \\
& $\textbf{VSR}^{\ddagger}$ & \colorbox{green!}{\phantom{-}22.2} & \colorbox{green!}{\phantom{-1}\underline{1.0}} & \colorbox{green!15}{\phantom{-1}0.4} & \colorbox{red!50}{\phantom{1}-3.9} & \colorbox{green!15}{\phantom{-1}0.1} & \colorbox{red!6}{\phantom{1}-0.2} & \colorbox{green!}{\phantom{-}19.5} & \colorbox{red!6}{\phantom{1}-0.3} & \colorbox{green!}{\phantom{-}\underline{23.3}} & \colorbox{red!70}{\phantom{1}-5.1} & \colorbox{green!}{\phantom{-}22.9} & \colorbox{red!12}{\phantom{1}-1.6} \\
& $\textbf{HM}^{\ddagger}$ & \colorbox{green!}{\phantom{-}10.6} & \colorbox{red!25}{\phantom{1}-2.2} & \colorbox{red!25}{\phantom{1}-1.8} & \colorbox{red!70}{\phantom{1}-5.8} & \colorbox{green!15}{\phantom{-1}0.7} & \colorbox{green!}{\phantom{-1}\underline{0.2}} & \colorbox{green!}{\phantom{-}10.7} & \colorbox{red!6}{\phantom{1}-0.1} & \colorbox{green!}{\phantom{-}\underline{11.7}} & \colorbox{red!12}{\phantom{1}-1.4} & \colorbox{green!}{\phantom{-}10.9} & \colorbox{red!6}{\phantom{1}-0.2} \\
& $\textbf{VisOnly}^{\ddagger}$ & \colorbox{red!50}{\phantom{1}-2.3} & \colorbox{green!}{\phantom{-1}\underline{0.7}} & \colorbox{red!25}{\phantom{1}-1.0} & \colorbox{red!70}{\phantom{1}-4.7} & \colorbox{green!15}{\phantom{-1}0.2} & \colorbox{green!50}{\phantom{-1}0.1} & \colorbox{red!25}{\phantom{1}-2.0} & \colorbox{green!}{\phantom{-1}0.5} & \colorbox{red!25}{\phantom{1}-1.0} & \colorbox{green!}{\phantom{-1}0.2} & \colorbox{green!15}{\phantom{-1}\underline{1.7}} & \colorbox{green!}{\phantom{-1}0.5} \\
\midrule
\multirow{8}{*}{\textbf{CL-50}} & $\textbf{GTS}^{\dagger}$ & \colorbox{red!25}{\phantom{1}-0.7} & \colorbox{red!6}{\phantom{1}\underline{-0.3}} & \colorbox{green!30}{\phantom{-1}3.1} & \colorbox{red!70}{-11.1} & \colorbox{green!}{\phantom{-1}\underline{9.7}} & \colorbox{red!12}{\phantom{1}-1.3} & \colorbox{red!25}{\phantom{1}-1.4} & \colorbox{red!70}{\phantom{1}-6.7} & \colorbox{red!70}{\phantom{1}-3.9} & \colorbox{red!25}{\phantom{1}-2.1} & \colorbox{green!}{\phantom{-1}6.9} & \colorbox{red!6}{\phantom{1}-0.4} \\
& $\textbf{TSI}^{\dagger}$ & \colorbox{green!}{\phantom{-1}9.9} & \colorbox{red!6}{\phantom{1}\underline{-0.0}} & \colorbox{green!}{\phantom{-}18.2} & \colorbox{red!70}{\phantom{1}-6.3} & \colorbox{green!}{\phantom{-}19.1} & \colorbox{red!6}{\phantom{1}-0.4} & \colorbox{red!25}{\phantom{1}-1.6} & \colorbox{red!70}{-16.5} & \colorbox{green!}{\phantom{-}15.1} & \colorbox{red!12}{\phantom{1}-0.7} & \colorbox{green!}{\phantom{-}\underline{22.0}} & \colorbox{red!12}{\phantom{1}-1.1} \\
& $\textbf{CAn}^{\dagger}$ & \colorbox{red!25}{\phantom{1}-1.8} & \colorbox{red!12}{\phantom{1}-0.7} & \colorbox{red!25}{\phantom{1}-0.4} & \colorbox{red!70}{\phantom{1}-4.4} & \colorbox{green!15}{\phantom{-1}\underline{1.3}} & \colorbox{red!6}{\phantom{1}\underline{-0.5}} & \colorbox{red!25}{\phantom{1}-1.8} & \colorbox{red!70}{\phantom{1}-9.8} & \colorbox{red!50}{\phantom{1}-2.1} & \colorbox{red!12}{\phantom{1}-1.1} & \colorbox{green!15}{\phantom{-1}1.0} & \colorbox{red!50}{\phantom{1}-3.4} \\
& $\textbf{AIR}^{\dagger}$ & \colorbox{green!30}{\phantom{-1}4.6} & \colorbox{green!}{\phantom{-1}\underline{0.4}} & \colorbox{green!}{\phantom{-1}7.8} & \colorbox{red!50}{\phantom{1}-3.8} & \colorbox{green!}{\phantom{-1}8.2} & \colorbox{red!12}{\phantom{1}-0.7} & \colorbox{green!}{\phantom{-1}6.2} & \colorbox{red!50}{\phantom{1}-3.1} & \colorbox{green!}{\phantom{-}\underline{17.9}} & \colorbox{red!12}{\phantom{1}-0.9} & \colorbox{green!}{\phantom{-1}8.9} & \colorbox{red!6}{\phantom{1}-0.4} \\
& $\textbf{ESAT}^{\dagger}$ & \colorbox{green!15}{\phantom{-1}1.0} & \colorbox{green!}{\phantom{-1}\underline{0.2}} & \colorbox{red!70}{-14.5} & \colorbox{red!50}{\phantom{1}-3.6} & \colorbox{green!}{\phantom{-1}\underline{7.0}} & \colorbox{green!50}{\phantom{-1}0.2} & \colorbox{green!15}{\phantom{-1}1.7} & \colorbox{green!50}{\phantom{-1}0.2} & \colorbox{red!70}{\phantom{1}-9.5} & \colorbox{red!6}{\phantom{1}-0.6} & \colorbox{red!25}{\phantom{1}-0.7} & \colorbox{red!6}{\phantom{1}-0.5} \\
& $\textbf{VSR}^{\ddagger}$ & \colorbox{green!}{\phantom{-}21.9} & \colorbox{green!}{\phantom{-1}1.0} & \colorbox{green!15}{\phantom{-1}0.4} & \colorbox{red!70}{\phantom{1}-4.5} & \colorbox{green!15}{\phantom{-1}2.3} & \colorbox{red!6}{\phantom{1}-0.3} & \colorbox{green!}{\phantom{-}20.2} & \colorbox{red!70}{\phantom{1}-5.3} & \colorbox{green!}{\phantom{-}21.0} & \colorbox{green!}{\phantom{-1}\underline{1.1}} & \colorbox{green!}{\phantom{-}\underline{23.4}} & \colorbox{red!50}{\phantom{1}-3.6} \\
& $\textbf{HM}^{\ddagger}$ & \colorbox{green!}{\phantom{-}10.2} & \colorbox{red!25}{\phantom{1}-2.1} & \colorbox{green!15}{\phantom{-1}0.7} & \colorbox{red!70}{\phantom{1}-4.5} & \colorbox{green!15}{\phantom{-1}0.3} & \colorbox{green!50}{\phantom{-1}0.2} & \colorbox{green!}{\phantom{-}12.5} & \colorbox{red!12}{\phantom{1}-1.5} & \colorbox{green!}{\phantom{-}\underline{12.3}} & \colorbox{red!50}{\phantom{1}-3.7} & \colorbox{green!}{\phantom{-}12.2} & \colorbox{green!}{\phantom{-1}\underline{0.2}} \\
& $\textbf{VisOnly}^{\ddagger}$ & \colorbox{red!50}{\phantom{1}-2.4} & \colorbox{green!}{\phantom{-1}0.6} & \colorbox{red!25}{\phantom{1}-0.2} & \colorbox{red!70}{\phantom{1}-6.8} & \colorbox{green!15}{\phantom{-1}0.3} & \colorbox{red!6}{\phantom{1}-0.1} & \colorbox{red!25}{\phantom{1}-2.0} & \colorbox{green!}{\phantom{-1}\underline{0.7}} & \colorbox{green!15}{\phantom{-1}0.2} & \colorbox{green!}{\phantom{-1}0.2} & \colorbox{green!15}{\phantom{-1}\underline{0.3}} & \colorbox{green!25}{\phantom{-1}0.1} \\
\bottomrule
\end{tabular}
\endgroup
\end{small}
\end{center}
\vskip -0.1in
\end{table*}

\textbf{Additional metrics.}~~We assess the performance of LoRSU against LoRA and SPU in terms of ACC and BWT across two out-of-domain datasets, GTS and ESAT. Since LoRA and SPU have similar number of trainable parameters as LoRSU and competitive performance in our previous experiment, we choose those for comparison. Table~\ref{table:bwt_metrics_clip_reduced} shows that LoRSU's  performs well with respect to these metrics, following similar patterns as TI and CC in Table~\ref{table:clip_baselines_summary}. LoRSU achieves the best performance on ACC while exhibiting minimal forgetting with the least negative BWT values. Similar patterns are observed on extra datasets in appendix~\ref{sec_appx:extra_bwt_results}. 

\subsection{CLIP-based vs. Perplexity-based Updates}
Traditionally, LLMs and VLMs achieve impressive performance through fine-tuning with the perplexity loss. LoRA is the standard PEFT method for this purpose, and thus, we consider three extra LoRA variants plus \emph{LoRSU-Ppl} which all utilize the perplexity loss to update the model. 
\begin{itemize}[noitemsep,topsep=1pt,parsep=1pt,partopsep=1pt,leftmargin=*]
    \item \emph{LoRA-L} applies LoRA adapters to all weight matrices of the LLM and thus perplexity loss is required.
    \item \emph{LoRA-Ppl} is the same method as LoRA but this time the perplexity loss is used to update the adapters.
    \item \emph{LoRA-F} applies LoRA adapters to all weight matrices of the LLM, the image encoder, and the MLP projector.
\end{itemize}

We aevaluate how LoRSU and LoRA perform compared to their perplexity-based counterparts, LoRSU-Ppl and LoRA-Ppl, respectively. Furthermore, we seek to explore how these methods compare to parameter-efficient fine-tuning approaches when either the entire VLM (LoRA-F) or only the LLM component (LoRA-L) is updated.

The results in Table \ref{table:ppl_vs_clip_summary} highlight the strong and robust performance of LoRSU and LoRSU-Ppl compared to other baseline methods across various settings. Both LoRSU and LoRSU-Ppl achieve minimal negative or even positive changes in CC, indicating reduced catastrophic forgetting and improved retention of generic knowledge compared to baselines. The table reports the average accuracies of TI/CC over three runs with exact results provided in appendix~\ref{sec_appx:detailed_res}.

The use of the perplexity loss in LoRSU-Ppl demonstrates a considerable improvement in TI accuracy over LoRSU when fine-tuned for VQA datasets. For instance, LoRSU-Ppl achieves 10\% higher TI accuracy than LoRSU on VSR. We hypothesize that the perplexity loss acts as an additional signal that optimizes the image encoder to complement the frozen language model more effectively, improving the alignment between visual and textual modalities in VQA.

However, we observe that LoRSU achieves a balance between task-specific improvements and generalization, consistently demonstrating higher CC accuracy compared to LoRSU-Ppl across most datasets. Lastly, although LoRA-F achieves high TI scores on many datasets, it suffers significantly from forgetting, underscoring the importance of LoRSU's structured updates in CL scenarios.

\subsection{The Choice of Attention Heads}
\begin{table}
\caption{Comparison of the importance of choosing a small subset of attention heads. The GTS dataset is used for fine-tuning. We include error bars over 3 runs. The highest accuracies across methods are underlined.}
 \label{table:lorsu_attn_vs_rand_summary}
\vskip 0.15in
\begin{center}
\begin{small}
\begingroup
\setlength{\tabcolsep}{3.5pt}
\begin{tabular}{l c c c c}
\toprule
\textbf{Setting} & \textbf{Scores} & \textbf{LoRSU-Rand} & \textbf{LoRSU-AAH} & \textbf{LoRSU} \\
\midrule
\multirow{2}{*}{\textbf{CL-5}} & \textbf{TI ($\uparrow)$}  & 4.1 \mtiny{\pm 0.4} & 5.9 \mtiny{\pm 0.8} & \underline{6.4 \mtiny{\pm 1.3}} \\
 & \textbf{CC ($\uparrow)$}  & -1.0 \mtiny{\pm 0.5} & -0.9 \mtiny{\pm 0.3} & \underline{-0.7 \mtiny{\pm 0.6}} \\
\midrule
\multirow{2}{*}{\textbf{CL-20}} & \textbf{TI ($\uparrow)$}  & 6.2 \mtiny{\pm 0.6} & 7.5 \mtiny{\pm 0.6} & \underline{8.6 \mtiny{\pm 0.9}} \\
 & \textbf{CC ($\uparrow)$}  & -1.4 \mtiny{\pm 0.3} & \underline{-0.7 \mtiny{\pm 0.4}} & -1.0 \mtiny{\pm 0.5} \\
\midrule
\multirow{2}{*}{\textbf{CL-50}} & \textbf{TI ($\uparrow)$}  & 7.8 \mtiny{\pm 0.4} & 9.1 \mtiny{\pm 0.1} & \underline{9.7 \mtiny{\pm 0.1}} \\
 & \textbf{CC ($\uparrow)$}  & -1.7 \mtiny{\pm 0.2} & \underline{-0.9 \mtiny{\pm 0.2}} & -1.3 \mtiny{\pm 0.1} \\
\bottomrule
\end{tabular}
\endgroup
\end{small}
\end{center}
\vskip -0.1in
\end{table}
Given that LoRSU's mechanism of choosing attention heads is a key-point to its success, we conduct an ablation study on the different strategies for selecting attention heads during the fine-tuning process. In this experiment, we compare LoRSU's performance to two new variants of LoRSU, namely, \emph{LoRSU-Rand} and \emph{LoRSU-AAH}. LoRSU-Rand randomly chooses the number of attention heads ($k=2$ heads) to be fine-tuned while LoRSU-AAH fine-tunes all the available attention heads (16 in total) in each transformer block. For extra results on the sensitivity of the number of LoRSU's optimal attention heads $k$, see appendix~\ref{sec_appx:ablation_heads_general}.

The results in Table~\ref{table:lorsu_attn_vs_rand_summary} demonstrate that  LoRSU's targeted approach is performant, balancing task-specific improvements (TI) and the retention of generic knowledge (CC).
Random selection (LoRSU-Rand) fails to generalize well, while fine-tuning all attention heads (LoRSU-AHH) adds unnecessary computational overhead with less effective generalization. LoRSU outperforms both of the variants in TI while LoRSU-AHH is marginally better in CC. Additional experiments that investigate the robustness of \ours in terms of the number of training epochs can be found in appendix~\ref{sec_appx:robustness}; ablation studies of other hyperparameters of LoRSU are given in appendix~\ref{sec_appx:extra_ablations}.

\begin{figure}
\vskip 0.2in
\centering
\includegraphics[width=0.48\textwidth]{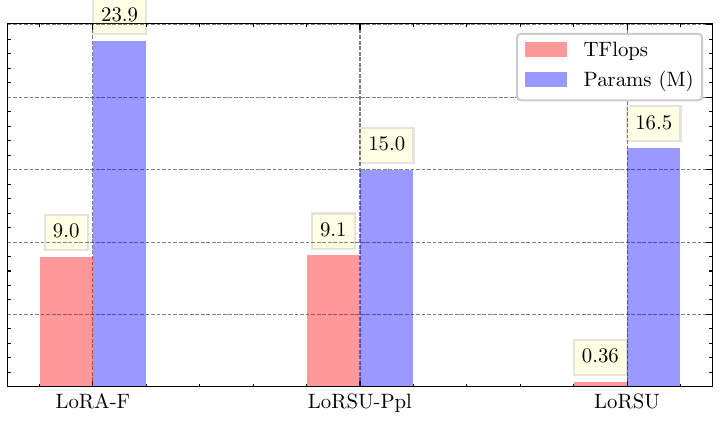}
\caption{TFlops and trainable parameters comparison between LoRSU with CLIP loss (LoRSU), perplexity loss (LoRSU-Ppl), and LoRA-F.}
\label{fig:tflops_comparisons}
\vskip -0.2in
\end{figure}
\subsection{Computational Efficiency}
In Figure \ref{fig:tflops_comparisons}, we asses the computational benefits of LoRSU using the CLIP loss comparing to baseline methods. We focus on two key metrics: trainable parameters and floating-point operations per second (TFLOPs). 

\ours requires 25× fewer computation resources than LoRA-F and LoRSU-Ppl, demonstrating the suitability of using CLIP loss when computational resources are limited. Unlike the perplexity loss, which necessitates forward and backward passes through both the vision encoder and LLM, the CLIP loss operates solely on the vision encoder, significantly reducing computational overhead. This makes LoRSU more scalable, enabling efficient continual learning even in resource-constrained environments.

\section{Discussion}\label{sec:discussion}
In this work, we introduced LoRSU, a novel parameter-efficient fine-tuning method, specifically designed for few-shot continual learning scenarios with VLMs. Unlike existing approaches, LoRSU operates without relying on a replay buffer, making it uniquely suited for resource-constrained settings. Through extensive experiments, we demonstrate that LoRSU achieves both computational efficiency and the preservation of the model’s generic knowledge by using localized and structured updates. LoRSU outperforms 12 baselines in over 80\% of evaluations across 10 datasets and 3 settings, achieving the highest TI accuracies in most cases while maintaining stable or even positive CC accuracies. To the best of our knowledge, we are the first to explore few-shot continual learning of VLMs.

Whilst we focus on CLIP and LLaVA due to computational constraints, our method is generic to any transformer model, and we plan to extend it to other VLMs and image encoders. Another promising direction is using a smaller LLM proxy model in perplexity-based methods like LoRSU-Ppl, which has shown strong VQA performance. This could improve scalability and LoRSU's use in resource-limited settings. Finally, LoRSU’s binary mask-based structured updates ensure efficient, precise parameter updates, but scaling to larger architectures like LLMs poses challenges. Replacing binary masks with more scalable solutions for vast parameter spaces will be crucial to manage memory and processing demands, offering opportunities for further refinement.

\section*{Impact Statement}

Authors are \textbf{required} to include a statement of the potential 
broader impact of their work, including its ethical aspects and future 
societal consequences. This statement should be in an unnumbered 
section at the end of the paper (co-located with Acknowledgements -- 
the two may appear in either order, but both must be before References), 
and does not count toward the paper page limit. In many cases, where 
the ethical impacts and expected societal implications are those that 
are well established when advancing the field of Machine Learning, 
substantial discussion is not required, and a simple statement such 
as the following will suffice:

``This paper presents work whose goal is to advance the field of 
Machine Learning. There are many potential societal consequences 
of our work, none which we feel must be specifically highlighted here.''

The above statement can be used verbatim in such cases, but we 
encourage authors to think about whether there is content which does 
warrant further discussion, as this statement will be apparent if the 
paper is later flagged for ethics review.

\bibliography{refs}
\bibliographystyle{icml2025}

\newpage
\appendix
\onecolumn

\section{Proof of the optimal mask $\bp^*$}\label{sec_appx:proof_mask}
\begin{definition}\label{def:tops}
The operator TOP-$C: \mathbb{R}^d \rightarrow \mathbb{R}^d$, for $1 \leq C \leq d$ is defined as
\begin{equation}\
    \left( \text{TOP-}C(\bx) \right)_{\pi(i)} := \left\{
\begin{array}{ll}
      x_{\pi(i)}, & i \leq C \\
      0, & \text{otherwise},\\
\end{array} 
\right. \nonumber
\end{equation}
where $\bx = (x_1, \ldots, x_d)^{\top} \in \RR^d$ and $\pi$ is a permutation of $\{1, 2, \ldots, d \}$ such that $|x_{\pi(i)}| \geq |x_{\pi(i+1)}|$, for $i=1, \ldots, d-1$, i.e. the TOP-$S$ operator keeps only the $S$ largest elements of $\bx$ in magnitude and truncates the rest to zero.
\end{definition}

\begin{lemma}\label{lemma:max_norm}
For any $\bx \in \mathbb{R}^d-\{ \mathbf{0}\}$, $1 \leq C \leq d$, the optimal mask
\begin{align}
     \bp^* & =  \argmax_{\bp \in \{0, 1 \}^d} \frac{\norm{\bp \odot \bx }^2}{\norm{\bx}^2},~~\text{s.t.}~\norm{\bp}_0 \leq C,  \nonumber
\end{align}
has zeros everywhere except the $C$ largest elements of $\bx$ in magnitude.
\begin{proof}
    Rewriting the optimization problem as 
    \begin{equation}
        \max_{\bp \in \{0, 1 \}^d} \sum_{i=1}^d p_i x_i^2,~~ \text{s.t.}~\sum_{i=1}^d p_i \leq C \nonumber,
    \end{equation}
    Notice that this is a trivial binary knapsack problem with maximum weight capacity $C$ and weights equal to one. Hence, the maximum is attained when we pick the top $C$ maximal $x_i^2$ elements. 
\end{proof}
\end{lemma}

\begin{remark}

\end{remark} It holds that $\text{TOP-}S(\bx) = \bp^* \odot \bx$.

\begin{corollary}
The optimal mask $\bp^*$ in \eqref{eq:maxim} has zeros everywhere except for the indices $i \in \{j: \exists \ell \in \{1, \ldots, G \},~\text{such that}~j \in \{ \pi_{\ell}(1), \ldots,  \pi_{\ell}(c_{\ell})\}    \}$, where $\pi_{\ell}$ is the same permutation as in Definition $\ref{def:tops}$ for the set of indices $I_{\ell}
$.
\end{corollary}
\begin{proof}
    The result follows from the mutual exclusiveness of $I_{\ell}$ in the constraints of \eqref{eq:maxim} and Lemma \ref{lemma:max_norm}.
\end{proof}

\section{Implementation Details}\label{sec_appx:implementation_details}
We describe below the implementation details of section~\ref{sec:experiments}.
\begin{itemize}
    \item All the experiments are conducted on a single NVIDIA A100 GPU.
    \item We have included error bars over three runs for all experiments.
    \item We use PyTorch~\cite{paszke2019pytorch} to implement all the algorithms.
    \item We use Adam~\citep{kingma2014adam} as an optimizer for the methods that utilize the CLIP loss for fine tuning and AdamW~\citep{loshchilov2017decoupled} for those ones that use the perplexity loss. 
    \item A learning rate scheduler of Cosine Annealing with Warmup is employed for all methods.
    \item For all experiments, we set the learning rate $1 \times 10^{-5}$ and $2 \times 10^{-5}$, for LoRSU and LoRSU-Ppl, respectively.
    \item We set batch size to 16 for all methods that fine-tune the vision encoder through CLIP loss. We reduce the batch size to 8 for those methods that fine-tune the vision encoder through perplexity loss or those that fine-tune the LLM. This was due to GPU memory limitations.
    \item All methods run for 20, 15, and 10 epochs for the CL-5, CL-10, and CL-50 settings, respectively.
    \item For LoRA (-Ppl), we set rank $r=64$ while LoRA-L and LoRA-F use $r=8$, for all experiments.
    \item For AdaLoRA, we set the initial rank to 70 and the final average rank to 64.
    \item The adapters of LoRA and AdaLoRA are applied to all weight matrices of each of the transformer blocks.
    \item For SPU, we use sparsity=15\% for all experiments.
    \item For \ours (-Ppl) we use sparsity=10\%, rank=64, and we pick the top-2 attention heads for all experiments.   
\end{itemize}
The choice of the above hyperparameters ensures that LoRA (-Ppl), LoRA-L, LoRA-F, AdaLoRA. SPU, and LoRSU (-Ppl) have similar number of trainable parameters. 

\section{Datasets}\label{sec_appx:datasets}
Details on all datasets used in section~\ref{sec:experiments} are presented here. 

\subsection{TSI \& DALLE}
We start with the description of how we constructed our newly introduced VQA datasets \emph{TSI} and \emph{DALLE}.

\paragraph{TSI.}  To extract  images from the videos of the Toyota Smart Home dataset (TSI), we discretized each video clip into 2 frames per second and then selected the frame in the middle of the total time duration of the video clip. In Table~\ref{table:tsi_class_names} we describe the actions that were selected and the corresponding prompt used for CLIP classification. We also note dropping few actions to avoid ambiguous classes.
\begin{table}
\caption{The original action names of the Toyota Smarthome dataset and their corresponding captions used to create the Toyota Smarthome Images (TSI) dataset. We use~\xmark~to denore the actions that are ambiguous and were not used to build the TSI dataset. The final prompt is created as ``\textit{The person in this image is \{caption\}}''.}
\label{table:tsi_class_names}
\vskip 0.15in
\begin{center}
\begingroup
\begin{tabular}{l c }
\toprule
\textbf{Original Class name/Action} & \textbf{Generated Caption}  \\
\midrule
Cook.Cleandishes & washing dishes \\
Cook.Cleanup & cleaning up \\
Cook.Cut & cutting food \\
Cook.Stir & stirring the pot \\
Cook.Usestove & \xmark \\
Cook.Cutbread & cutting bread \\
Drink.Frombottle & holding a bottle \\
Drink.Fromcan & holding a can \\
Drink.Fromcup & holding a cup \\
Drink.Fromglass & holding a glass \\
Eat.Attable & eating \\
Eat.Snack & \xmark \\
Enter & walking \\
Getup & \xmark \\
Laydown & lying down \\
Leave & walking \\
Makecoffee.Pourgrains & using a white coffee machine \\
Makecoffee.Pourwater & using a white coffee machine \\
Maketea.Boilwater & boiling water in a black kettle \\
Maketea.Insertteabag & making tea \\
Pour.Frombottle & holding a bottle \\
Pour.Fromcan & holding a can \\
Pour.Fromkettle & holding a black kettle \\
Readbook & reading a book \\
Sitdown & sitting down \\
Takepills & \xmark \\
Uselaptop & using a laptop \\
Usetablet & using a tablet \\
Usetelephone & using a cordless phone \\
Walk & walking \\
WatchTV & watching TV \\
\bottomrule
\end{tabular}
\endgroup
\end{center}
\vskip -0.1in
\end{table}

\paragraph{DALLE.} We generated images from DALL·E 2 using OpenAI python package and we used the prompt
 ``\textit{A person} $\{a\}$'' where $a \in $ \{ \textit{using a white coffee machine, 
                 eating, 
                 cutting bread, 
                 stirring the pot, 
                 holding a glass, 
                 watching TV, 
                 holding a bottle, 
                 walking, 
                 making tea, 
                 cutting food, 
                 holding a cup, 
                 using a laptop, 
                 lying down, 
                 holding a can, 
                 person holding a black kettle, 
                 reading a book, 
                 cleaning up, 
                 sitting down, 
                 using a tablet, 
                 boiling water in a black kettle, 
                 using a cordless phone, 
                 washing dishes}\}.

In Table~\ref{table:datasets_num_data}, we present the average number of images per session used to update the model for each CL setting. Finally, Table~\ref{table:datasets_num_data_evaluation} provides characteristics of the datasets used for evaluating performance.

\subsection{Continual Learning Splits}
For the continual learning settings of section~\ref{sec:experiments}, we split all datasets into five non-overlapping continual learning (CL) splits based on the classes/categories of each dataset. Unless stated otherwise, we use the training split of each dataset to construct these CL splits.

\paragraph{GTS~\cite{stallkamp2012man}.} We split the 43 classes of GTS as follows:
\begin{itemize}
    \item \emph{Session 1:} $\left[ 25, 2, 11,  1, 40, 27,  5,  9, 17 \right]$.
    \item \emph{Session 2:} $\left[ 32, 29, 20, 39, 21, 15, 23, 10, 3 \right]$.
    \item \emph{Session 3:} $\left[ 18, 38, 42, 14, 22, 35, 34, 19, 33 \right]$.
    \item \emph{Session 4:} $\left[ 12, 26, 41, 0, 37, 6, 13, 24 \right]$.
    \item \emph{Session 5:} $\left[ 30, 28, 31, 7, 16, 4, 36, 8 \right]$.
\end{itemize}

\paragraph{TSI~\cite{das2019toyota}.} We split the 27 action categories of TSI as follows:
\begin{itemize}
    \item \emph{Session 1:} [\textit{WatchTV}, \textit{Laydown}, \textit{Sitdown}, \textit{Pour.Fromkettle}, \textit{Enter}, \textit{Drink.Frombottle}].
    \item \emph{Session 2:} [\textit{Eat.Attable}, \textit{Pour.Frombottle}, \textit{Cook.Cleandishes}, \textit{Maketea.Boilwater}, \textit{Leave}, \textit{Cook.Cleanup}].
    \item \emph{Session 3:} [\textit{Maketea.Insertteabag}, \textit{Makecoffee.Pourwater}, \textit{Drink.Fromcan}, \textit{Readbook}, \textit{Cutbread}].
    \item \emph{Session 4:} [\textit{Drink.Fromcup}, \textit{Drink.Fromglass}, \textit{Usetablet}, \textit{Pour.Fromcan}, \textit{Usetelephone}].
    \item \emph{Session 5:} [\textit{Walk}, \textit{Cook.Stir}, \textit{Makecoffee.Pourgrains}, \textit{Cook.Cut}, \textit{Uselaptop}].
\end{itemize}

\paragraph{CAn~\cite{wang2024clips}.} The 45 classes of CAn are split as follows:
\begin{itemize}
    \item \emph{Session 1:} $\left[ 102, 9, 20, 56, 23, 30, 357, 291, 144 \right]$.
    \item \emph{Session 2:} $\left[ 41, 293, 42, 49, 54, 57, 70, 279, 305 \right]$.
    \item \emph{Session 3:} $\left[ 71, 10, 76, 79, 349, 16, 81, 83, 100 \right]$.
    \item \emph{Session 4:} $\left[ 130, 30, 133, 150, 275, 276, 58, 277, 80 \right]$.
    \item \emph{Session 5:} $\left[ 39, 290, 37, 296, 316, 337, 89, 360, 128 \right]$.
\end{itemize}
The indices of CAn correspond to those of ImageNet~\cite{imagenet} since the dataset was built based on these 45 animal classes of ImageNet.

\paragraph{AIR~\cite{maji13fine-grained}.} We split the 100 aircraft types of AIR as follows:
\begin{itemize}
    \item \emph{Session 1:} $\left[ 23, 8, 11, 7, 48, 13, 1, 91, 94, 54, 16, 63, 52, 41, 80, 2, 47, 87, 78, 66 \right]$.
    \item \emph{Session 2:} $\left[ 19, 6, 24, 10, 59, 30, 22, 29, 83, 37, 93, 81, 43, 99, 86, 28, 34, 88, 44, 14 \right]$.
    \item \emph{Session 3:} $\left[ 84, 70, 4, 20, 15, 21, 31, 76, 57, 67, 73, 50, 69, 25, 98, 46, 96, 0, 72, 35 \right]$.
    \item \emph{Session 4:} $\left[ 58, 92, 3, 95, 56, 90, 26, 40, 55, 89, 75, 71, 60, 42, 9, 82, 39, 18, 77, 68 \right]$.
    \item \emph{Session 5:} $\left[ 32, 79, 12, 85, 36, 17, 64, 27, 74, 45, 61, 38, 51, 62, 65, 33, 5, 53, 97, 49 \right]$.
\end{itemize}

\paragraph{ESAT~\cite{helber2019eurosat}.} We split the 10 different land terrain classes of ESAT as follows:
\begin{itemize}
    \item \emph{Session 1:} $\left[ 0, 1 \right]$.
    \item \emph{Session 2:} $\left[ 2, 3 \right]$.
    \item \emph{Session 3:} $\left[ 4, 5 \right]$.
    \item \emph{Session 4:} $\left[ 6, 7 \right]$.
    \item \emph{Session 5:} $\left[ 8, 9 \right]$.
\end{itemize}

\paragraph{DALLE.} This dataset was only used for performance evaluation (control dataset), and not fine-tuning.

\paragraph{VSR~\cite{Liu2022VisualSR}.} The images of this VQA dataset are labeled according to 36 different categories that describe the dominant object of the image. We create the CL splits as follows:
\begin{itemize}
    \item \emph{Session 1:} [\textit{oven}, \textit{dining table}, \textit{spoon}, \textit{boat}, \textit{cake}, \textit{donut}, \textit{sandwich}].
    \item \emph{Session 2:} [\textit{fire hydrant}, \textit{elephant}, \textit{airplane}, \textit{truck}, \textit{apple}, \textit{hot dog}, \textit{sheep}].
    \item \emph{Session 3:} [\textit{kite}, \textit{baseball glove}, \textit{cow}, \textit{tie}, \textit{scissors}, \textit{toaster}, \textit{tv}].
    \item \emph{Session 4:} [\textit{bicycle}, \textit{banana}, \textit{couch}, \textit{teddy bear}, \textit{bus}, \textit{umbrella}, \textit{bird}].
    \item \emph{Session 5:} [\textit{potted plant}, \textit{bowl}, \textit{broccoli}, \textit{bottle}, \textit{knife}, \textit{orange}, \textit{person}, \textit{pizza}].
\end{itemize}

\paragraph{HM~\cite{kiela2020hateful}.} For the hateful memes dataset, since there was not any labeling information of the images so we can spli the images in a meaningful way, we randomly split the training images into five disjoint sets to create our final CL splits.

\paragraph{MMVP~\cite{tong2024eyes}.} This is the only dataset where no training split is available and it is comprised of just 300 images. For this reason, we only used it for evaluation in our experiments in the main paper. However, for completeness, we included results in Table~\ref{table:fine_tune_llm_mmvp} where we fine-tune on it. We use 150 images for training which are equally split into five sessions and the rest of the 150 images are used for evaluation. Thus, the setting can be considered as a 30-shot CL setting. 

\paragraph{VisOnly~\cite{kamoi2024visonlyqa}.} This dataset categorizes its samples into seven categories describing the nature of the geometric and numerical information in scientific figures. We created the splits as follows:
\begin{itemize}
    \item \emph{Session 1:} \textit{Geometry-Triangle}.
    \item \emph{Session 2:} \textit{Geometry-Quadrilateral}.
    \item \emph{Session 3:} \textit{Geometry-Length}
    \item \emph{Session 4:} \textit{Geometry-Angle}.
    \item \emph{Session 5:} [\textit{Geometry-Area}, \textit{3D-Size}, \textit{3D-Angle}].
\end{itemize}

\begin{table}
\caption{Average number of images per session (5 sessions in total) for each dataset used for fine-tuning.}
 \label{table:datasets_num_data}
\vskip 0.15in
\begin{center}
\begingroup
\setlength{\tabcolsep}{9.7pt}
\begin{tabular}{l c c c c c c c c c}
\toprule
 & \multicolumn{8}{c}{\textbf{FT Dataset}}  \\
\cmidrule(lr){2-9}
\textbf{Setting} & \textbf{GTS} & \textbf{TSI} & \textbf{CAn} & \textbf{AIR} & \textbf{ESAT}  & \textbf{VSR} & \textbf{HM} & \textbf{VisOnly} \\
\midrule
\textbf{CL-5} & $43.0$ & $27.0$ & $45.0$ & $100.0$ & $10.0$ & $100.0$ & $100.0$ & $7.0$ \\
\midrule
 \textbf{CL-20} & $170.0$ & $84.0$ & $180.0$ & $400.0$ & $40.0$ & $274.6$ & $300.0$ & $28.0$ \\
\midrule
 \textbf{CL-50} & $430.0$ & $253.8$ & $450.0$ & $1000.0$ & $100.0$ & $485.2$ & $600.0$ & $70.0$ \\
\bottomrule
\end{tabular}
\endgroup
\end{center}
\end{table}

\begin{table}
\caption{Characteristics of the datasets used for performance  evaluation in section~\ref{sec:experiments}.}
 \label{table:datasets_num_data_evaluation}
\vskip 0.15in
\begin{center}
\begingroup
\setlength{\tabcolsep}{6.7pt}
\begin{tabular}{l c c c c c c c c c c c c c c}
\toprule
\textbf{Eval Datasets} & \textbf{GTS} & 
 \textbf{TSI} & \textbf{CAn} & \textbf{AIR} & \textbf{ESAT} & \textbf{DALLE} & \textbf{VSR} & \textbf{HM} & \textbf{MMVP} & \textbf{VisOnly} \\
\midrule
\textbf{\# Samples} & $3,990$ & $4,908$ & $1,796$ & $3,333$ & $17,000$ & $660$ & $1,222$ & $2,000$ & $150$ & $1,150$ \\
\textbf{\# Classes} & $43$ & $27$ & $45$ & $100$ & $10$ & $27$ & $36$ & NaN & NaN & $7$ \\
\bottomrule
\end{tabular}
\endgroup
\end{center}
\vskip -0.1in
\end{table}

\section{Detailed Results}\label{sec_appx:detailed_res}

\subsection{CLIP-based Updates+}
The detailed accuracies for all baselines and datasets used to create Table~\ref{table:clip_baselines_summary} of the main paper can be found in Tables~\ref{table:vlm_vqa_acc_gtsrb_clip} through~\ref{table:vlm_vqa_acc_eurosat_clip}.
\begin{table}
\caption{Accuracy scores (\%) for LLaVA with the pretrained (\emph{Zr-Shot}) or fine-tuned image encoder. All baselines use \emph{GTS} dataset for fine-tuning the image encoder~(the LLM remains frozen) via CLIP loss. We include error bars over 3 runs.}
 \label{table:vlm_vqa_acc_gtsrb_clip}
\vskip 0.15in
\begin{center}
\begin{small}
\begingroup
\setlength{\tabcolsep}{3.9pt}
\begin{tabular}{l c c c c c c c c c c c}
\toprule
 & & \multicolumn{9}{c}{\textbf{VQA Datasets (Acc \%)}}  \\
\cmidrule(lr){3-12}
\textbf{Setting} & \textbf{FT Method}  & \textbf{GTS} & \textbf{TSI} & \textbf{CAn} & \textbf{AIR} & \textbf{ESAT} & \textbf{DALLE} & \textbf{VSR} & \textbf{HM} & \textbf{MMVP} & \textbf{VisOnly} \\
\midrule
 & \textbf{Zr-Shot} & $75.6$ & $53.1$ & $82.7$ & $60.4$ & $76.1$ & $91.1$ & $51.5$ & $61.2$ & $58.0$ & $31.3$ \\
\midrule
\multirow{7}{*}{\textbf{CL-5}} & \textbf{LN} & $79.1 \mtiny{\pm 1.2}$ & $53.6 \mtiny{\pm 0.5}$ & $81.2 \mtiny{\pm 0.6}$ & $61.0 \mtiny{\pm 1.2}$ & $58.9 \mtiny{\pm 0.9}$ & $91.1 \mtiny{\pm 1.3}$ & $51.9 \mtiny{\pm 1.5}$ & $62.7 \mtiny{\pm 1.1}$ & $59.6 \mtiny{\pm 0.2}$ & $31.8 \mtiny{\pm 0.4}$ \\
& \textbf{F-FT} & $79.3 \mtiny{\pm 0.6}$ & $55.1 \mtiny{\pm 0.8}$ & $76.8 \mtiny{\pm 1.3}$ & $58.8 \mtiny{\pm 1.0}$ & $25.6 \mtiny{\pm 0.9}$ & $89.2 \mtiny{\pm 1.2}$ & $51.7 \mtiny{\pm 0.9}$ & $62.1 \mtiny{\pm 0.8}$ & $56.4 \mtiny{\pm 0.4}$ & $30.9 \mtiny{\pm 0.2}$ \\
& \textbf{F-EWC} & $80.6 \mtiny{\pm 0.6}$ & $37.4 \mtiny{\pm 1.3}$ & $63.2 \mtiny{\pm 0.7}$ & $55.8 \mtiny{\pm 1.4}$ & $26.1 \mtiny{\pm 1.4}$ & $81.5 \mtiny{\pm 1.1}$ & $51.8 \mtiny{\pm 1.4}$ & $61.2 \mtiny{\pm 0.6}$ & $53.8 \mtiny{\pm 0.4}$ & $31.2 \mtiny{\pm 0.4}$ \\
& \textbf{LoRA} & $76.3 \mtiny{\pm 0.8}$ & $52.6 \mtiny{\pm 1.4}$ & $73.3 \mtiny{\pm 0.6}$ & $56.7 \mtiny{\pm 1.2}$ & $49.3 \mtiny{\pm 0.8}$ & $87.1 \mtiny{\pm 1.3}$ & $51.8 \mtiny{\pm 1.2}$ & $61.3 \mtiny{\pm 1.2}$ & $58.1 \mtiny{\pm 0.3}$ & $31.6 \mtiny{\pm 0.4}$ \\
& \textbf{AdaLoRA} & $74.7 \mtiny{\pm 0.9}$ & $49.7 \mtiny{\pm 0.7}$ & $79.6 \mtiny{\pm 0.9}$ & $56.3 \mtiny{\pm 0.8}$ & $42.5 \mtiny{\pm 0.8}$ & $91.6 \mtiny{\pm 1.1}$ & $52.0 \mtiny{\pm 0.8}$ & $60.9 \mtiny{\pm 1.2}$ & $57.1 \mtiny{\pm 0.3}$ & $31.7 \mtiny{\pm 0.2}$ \\
\bottomrule
\end{tabular}
\endgroup
\end{small}
\end{center}
\vskip -0.1in
\end{table}

\begin{table}
\caption{Accuracy scores (\%) for LLaVA with the pretrained (\emph{Zr-Shot}) or fine-tuned image encoder. All baselines use \emph{TSI} dataset for fine-tuning the image encoder~(the LLM remains frozen) via CLIP loss. We include error bars over 3 runs.}
 \label{table:vlm_vqa_acc_tsi_clip}
\vskip 0.15in
\begin{center}
\begin{small}
\begingroup
\setlength{\tabcolsep}{3.9pt}
\begin{tabular}{l c c c c c c c c c c c}
\toprule
 & & \multicolumn{9}{c}{\textbf{VQA Datasets (Acc \%)}}  \\
\cmidrule(lr){3-12}
\textbf{Setting} & \textbf{FT Method}  & \textbf{GTS} & \textbf{TSI} & \textbf{CAn} & \textbf{AIR} & \textbf{ESAT} & \textbf{DALLE} & \textbf{VSR} & \textbf{HM} & \textbf{MMVP} & \textbf{VisOnly} \\
\midrule
 & \textbf{Zr-Shot} & $75.6$ & $53.1$ & $82.7$ & $60.4$ & $76.1$ & $91.1$ & $51.5$ & $61.2$ & $58.0$ & $31.3$ \\
\midrule
\multirow{7}{*}{\textbf{CL-5}} & \textbf{LN} & $75.4 \mtiny{\pm 1.0}$ & $53.9 \mtiny{\pm 0.6}$ & $82.6 \mtiny{\pm 1.3}$ & $60.0 \mtiny{\pm 1.0}$ & $75.9 \mtiny{\pm 0.8}$ & $91.1 \mtiny{\pm 1.3}$ & $51.7 \mtiny{\pm 1.4}$ & $61.9 \mtiny{\pm 1.0}$ & $58.4 \mtiny{\pm 0.3}$ & $30.9 \mtiny{\pm 0.3}$ \\
& \textbf{F-FT} & $73.8 \mtiny{\pm 0.5}$ & $60.5 \mtiny{\pm 1.1}$ & $81.6 \mtiny{\pm 0.9}$ & $59.5 \mtiny{\pm 1.5}$ & $70.4 \mtiny{\pm 1.0}$ & $91.1 \mtiny{\pm 1.2}$ & $51.8 \mtiny{\pm 0.9}$ & $61.5 \mtiny{\pm 1.3}$ & $56.9 \mtiny{\pm 0.2}$ & $31.3 \mtiny{\pm 0.3}$ \\
& \textbf{F-EWC} & $74.9 \mtiny{\pm 1.1}$ & $61.6 \mtiny{\pm 1.0}$ & $82.1 \mtiny{\pm 1.1}$ & $58.8 \mtiny{\pm 0.9}$ & $72.3 \mtiny{\pm 1.2}$ & $89.9 \mtiny{\pm 1.4}$ & $51.9 \mtiny{\pm 0.9}$ & $62.4 \mtiny{\pm 1.4}$ & $55.5 \mtiny{\pm 0.4}$ & $31.5 \mtiny{\pm 0.3}$ \\
& \textbf{LoRA} & $73.4 \mtiny{\pm 1.0}$ & $53.0 \mtiny{\pm 0.9}$ & $80.2 \mtiny{\pm 0.6}$ & $58.8 \mtiny{\pm 0.7}$ & $59.1 \mtiny{\pm 1.4}$ & $90.2 \mtiny{\pm 1.1}$ & $51.6 \mtiny{\pm 1.3}$ & $61.2 \mtiny{\pm 1.4}$ & $56.7 \mtiny{\pm 0.4}$ & $31.7 \mtiny{\pm 0.4}$ \\
& \textbf{AdaLoRA} & $75.6 \mtiny{\pm 0.8}$ & $54.2 \mtiny{\pm 0.6}$ & $82.6 \mtiny{\pm 1.1}$ & $60.0 \mtiny{\pm 1.3}$ & $75.7 \mtiny{\pm 1.3}$ & $91.1 \mtiny{\pm 1.2}$ & $51.6 \mtiny{\pm 0.9}$ & $62.1 \mtiny{\pm 1.0}$ & $59.5 \mtiny{\pm 0.3}$ & $31.7 \mtiny{\pm 0.2}$ \\
\bottomrule
\end{tabular}
\endgroup
\end{small}
\end{center}
\vskip -0.1in
\end{table}

\begin{table}
\caption{Accuracy scores (\%) for LLaVA with the pretrained (\emph{Zr-Shot}) or fine-tuned image encoder. All baselines use \emph{CAn} dataset for fine-tuning the image encoder~(the LLM remains frozen) via CLIP loss. We include error bars over 3 runs.}
 \label{table:vlm_vqa_acc_counteranimal_clip}
\vskip 0.15in
\begin{center}
\begin{small}
\begingroup
\setlength{\tabcolsep}{3.9pt}
\begin{tabular}{l c c c c c c c c c c c}
\toprule
 & & \multicolumn{9}{c}{\textbf{VQA Datasets (Acc \%)}}  \\
\cmidrule(lr){3-12}
\textbf{Setting} & \textbf{FT Method}  & \textbf{GTS} & \textbf{TSI} & \textbf{CAn} & \textbf{AIR} & \textbf{ESAT} & \textbf{DALLE} & \textbf{VSR} & \textbf{HM} & \textbf{MMVP} & \textbf{VisOnly} \\
\midrule
 & \textbf{Zr-Shot} & $75.6$ & $53.1$ & $82.7$ & $60.4$ & $76.1$ & $91.1$ & $51.5$ & $61.2$ & $58.0$ & $31.3$ \\
\midrule
\multirow{7}{*}{\textbf{CL-5}} & \textbf{LN} & $74.3 \mtiny{\pm 1.5}$ & $52.9 \mtiny{\pm 1.4}$ & $80.3 \mtiny{\pm 1.4}$ & $58.9 \mtiny{\pm 0.7}$ & $72.4 \mtiny{\pm 1.2}$ & $91.1 \mtiny{\pm 0.8}$ & $52.0 \mtiny{\pm 0.9}$ & $61.5 \mtiny{\pm 1.2}$ & $61.7 \mtiny{\pm 0.3}$ & $32.1 \mtiny{\pm 0.4}$ \\
& \textbf{F-FT} & $73.5 \mtiny{\pm 1.1}$ & $50.6 \mtiny{\pm 0.9}$ & $80.3 \mtiny{\pm 0.8}$ & $56.5 \mtiny{\pm 0.6}$ & $63.1 \mtiny{\pm 0.6}$ & $91.3 \mtiny{\pm 1.5}$ & $51.7 \mtiny{\pm 1.4}$ & $61.8 \mtiny{\pm 0.8}$ & $58.4 \mtiny{\pm 0.2}$ & $31.3 \mtiny{\pm 0.4}$ \\
& \textbf{F-EWC} & $65.9 \mtiny{\pm 1.5}$ & $39.1 \mtiny{\pm 0.7}$ & $66.0 \mtiny{\pm 1.3}$ & $40.0 \mtiny{\pm 0.9}$ & $41.7 \mtiny{\pm 0.7}$ & $86.2 \mtiny{\pm 0.8}$ & $51.8 \mtiny{\pm 1.3}$ & $59.9 \mtiny{\pm 1.0}$ & $57.6 \mtiny{\pm 0.4}$ & $31.3 \mtiny{\pm 0.2}$ \\
& \textbf{LoRA} & $69.7 \mtiny{\pm 1.4}$ & $44.8 \mtiny{\pm 1.1}$ & $81.4 \mtiny{\pm 0.7}$ & $56.9 \mtiny{\pm 1.0}$ & $50.7 \mtiny{\pm 1.3}$ & $92.9 \mtiny{\pm 1.3}$ & $52.0 \mtiny{\pm 1.0}$ & $61.8 \mtiny{\pm 1.5}$ & $56.5 \mtiny{\pm 0.4}$ & $31.3 \mtiny{\pm 0.4}$ \\
& \textbf{AdaLoRA} & $75.5 \mtiny{\pm 1.4}$ & $53.2 \mtiny{\pm 0.7}$ & $81.7 \mtiny{\pm 0.6}$ & $60.1 \mtiny{\pm 0.7}$ & $72.0 \mtiny{\pm 1.2}$ & $92.1 \mtiny{\pm 0.9}$ & $51.9 \mtiny{\pm 1.4}$ & $61.8 \mtiny{\pm 1.5}$ & $59.0 \mtiny{\pm 0.3}$ & $31.9 \mtiny{\pm 0.3}$ \\
\bottomrule
\end{tabular}
\endgroup
\end{small}
\end{center}
\vskip -0.1in
\end{table}

\begin{table}
\caption{Accuracy scores (\%) for LLaVA with the pretrained (\emph{Zr-Shot}) or fine-tuned image encoder. All baselines use \emph{AIR} dataset for fine-tuning the image encoder~(the LLM remains frozen) via CLIP loss. We include error bars over 3 runs.}
 \label{table:vlm_vqa_acc_aircraft_clip}
\vskip 0.15in
\begin{center}
\begin{small}
\begingroup
\setlength{\tabcolsep}{3.9pt}
\begin{tabular}{l c c c c c c c c c c c}
\toprule
 & & \multicolumn{9}{c}{\textbf{VQA Datasets (Acc \%)}}  \\
\cmidrule(lr){3-12}
\textbf{Setting} & \textbf{FT Method}  & \textbf{GTS} & \textbf{TSI} & \textbf{CAn} & \textbf{AIR} & \textbf{ESAT} & \textbf{DALLE} & \textbf{VSR} & \textbf{HM} & \textbf{MMVP} & \textbf{VisOnly} \\
\midrule
 & \textbf{Zr-Shot} & $75.6$ & $53.1$ & $82.7$ & $60.4$ & $76.1$ & $91.1$ & $51.5$ & $61.2$ & $58.0$ & $31.3$ \\
\midrule
\multirow{7}{*}{\textbf{CL-5}} & \textbf{LN} & $73.4 \mtiny{\pm 0.8}$ & $51.3 \mtiny{\pm 1.2}$ & $80.2 \mtiny{\pm 0.6}$ & $60.7 \mtiny{\pm 1.5}$ & $66.9 \mtiny{\pm 0.7}$ & $91.3 \mtiny{\pm 0.6}$ & $51.9 \mtiny{\pm 0.9}$ & $62.4 \mtiny{\pm 1.2}$ & $58.5 \mtiny{\pm 0.2}$ & $30.6 \mtiny{\pm 0.2}$ \\
& \textbf{F-FT} & $72.5 \mtiny{\pm 1.2}$ & $50.5 \mtiny{\pm 0.5}$ & $79.9 \mtiny{\pm 0.9}$ & $62.4 \mtiny{\pm 0.9}$ & $60.7 \mtiny{\pm 1.4}$ & $90.6 \mtiny{\pm 0.5}$ & $51.7 \mtiny{\pm 0.9}$ & $60.9 \mtiny{\pm 1.1}$ & $58.3 \mtiny{\pm 0.4}$ & $31.4 \mtiny{\pm 0.3}$ \\
& \textbf{F-EWC} & $74.9 \mtiny{\pm 1.2}$ & $52.4 \mtiny{\pm 0.8}$ & $71.5 \mtiny{\pm 1.2}$ & $63.3 \mtiny{\pm 1.0}$ & $63.8 \mtiny{\pm 1.0}$ & $90.7 \mtiny{\pm 1.5}$ & $51.2 \mtiny{\pm 0.5}$ & $61.2 \mtiny{\pm 0.8}$ & $58.1 \mtiny{\pm 0.4}$ & $31.4 \mtiny{\pm 0.4}$ \\
& \textbf{LoRA} & $70.9 \mtiny{\pm 0.9}$ & $52.7 \mtiny{\pm 0.6}$ & $79.0 \mtiny{\pm 0.7}$ & $61.7 \mtiny{\pm 0.5}$ & $48.8 \mtiny{\pm 0.7}$ & $90.6 \mtiny{\pm 0.6}$ & $52.0 \mtiny{\pm 0.9}$ & $62.5 \mtiny{\pm 0.8}$ & $60.0 \mtiny{\pm 0.3}$ & $31.1 \mtiny{\pm 0.2}$ \\
& \textbf{AdaLoRA} & $75.0 \mtiny{\pm 1.0}$ & $53.3 \mtiny{\pm 0.8}$ & $83.7 \mtiny{\pm 0.9}$ & $60.8 \mtiny{\pm 0.8}$ & $75.2 \mtiny{\pm 1.5}$ & $91.7 \mtiny{\pm 1.0}$ & $51.6 \mtiny{\pm 0.8}$ & $61.6 \mtiny{\pm 0.8}$ & $56.9 \mtiny{\pm 0.3}$ & $31.9 \mtiny{\pm 0.4}$ \\
\bottomrule
\end{tabular}
\endgroup
\end{small}
\end{center}
\vskip -0.1in
\end{table}

\begin{table}
\caption{Accuracy scores (\%) for LLaVA with the pretrained (\emph{Zr-Shot}) or fine-tuned image encoder. All baselines use \emph{ESAT} dataset for fine-tuning the image encoder~(the LLM remains frozen) via CLIP loss. We include error bars over 3 runs.}
 \label{table:vlm_vqa_acc_eurosat_clip}
\vskip 0.15in
\begin{center}
\begin{small}
\begingroup
\setlength{\tabcolsep}{3.9pt}
\begin{tabular}{l c c c c c c c c c c c}
\toprule
 & & \multicolumn{9}{c}{\textbf{VQA Datasets (Acc \%)}}  \\
\cmidrule(lr){3-12}
\textbf{Setting} & \textbf{FT Method}  & \textbf{GTS} & \textbf{TSI} & \textbf{CAn} & \textbf{AIR} & \textbf{ESAT} & \textbf{DALLE} & \textbf{VSR} & \textbf{HM} & \textbf{MMVP} & \textbf{VisOnly} \\
\midrule
 & \textbf{Zr-Shot} & $75.6$ & $53.1$ & $82.7$ & $60.4$ & $76.1$ & $91.1$ & $51.5$ & $61.2$ & $58.0$ & $31.3$ \\
\midrule
\multirow{7}{*}{\textbf{CL-5}} & \textbf{LN} & $75.8 \mtiny{\pm 0.9}$ & $53.2 \mtiny{\pm 0.6}$ & $82.6 \mtiny{\pm 1.1}$ & $60.0 \mtiny{\pm 1.3}$ & $80.3 \mtiny{\pm 1.0}$ & $92.7 \mtiny{\pm 1.0}$ & $51.9 \mtiny{\pm 0.7}$ & $61.7 \mtiny{\pm 0.5}$ & $60.4 \mtiny{\pm 0.4}$ & $31.8 \mtiny{\pm 0.3}$ \\
& \textbf{F-FT} & $69.1 \mtiny{\pm 0.8}$ & $50.5 \mtiny{\pm 0.6}$ & $80.8 \mtiny{\pm 1.1}$ & $57.7 \mtiny{\pm 1.5}$ & $65.8 \mtiny{\pm 0.6}$ & $91.3 \mtiny{\pm 1.5}$ & $51.8 \mtiny{\pm 0.7}$ & $62.0 \mtiny{\pm 0.7}$ & $58.8 \mtiny{\pm 0.3}$ & $30.4 \mtiny{\pm 0.2}$ \\
& \textbf{F-EWC} & $66.3 \mtiny{\pm 0.9}$ & $52.1 \mtiny{\pm 1.4}$ & $79.3 \mtiny{\pm 1.0}$ & $56.8 \mtiny{\pm 1.3}$ & $67.7 \mtiny{\pm 1.3}$ & $90.9 \mtiny{\pm 0.8}$ & $51.9 \mtiny{\pm 1.3}$ & $62.0 \mtiny{\pm 1.2}$ & $55.4 \mtiny{\pm 0.2}$ & $30.9 \mtiny{\pm 0.4}$ \\
& \textbf{LoRA} & $73.2 \mtiny{\pm 1.3}$ & $49.3 \mtiny{\pm 1.2}$ & $80.6 \mtiny{\pm 0.9}$ & $60.4 \mtiny{\pm 1.1}$ & $74.5 \mtiny{\pm 0.8}$ & $92.3 \mtiny{\pm 1.3}$ & $52.0 \mtiny{\pm 1.1}$ & $61.6 \mtiny{\pm 1.1}$ & $57.4 \mtiny{\pm 0.4}$ & $31.4 \mtiny{\pm 0.3}$ \\
& \textbf{AdaLoRA} & $75.9 \mtiny{\pm 0.5}$ & $52.4 \mtiny{\pm 1.4}$ & $82.4 \mtiny{\pm 0.5}$ & $60.5 \mtiny{\pm 0.8}$ & $78.0 \mtiny{\pm 1.3}$ & $91.5 \mtiny{\pm 0.9}$ & $51.6 \mtiny{\pm 0.8}$ & $61.5 \mtiny{\pm 1.3}$ & $59.0 \mtiny{\pm 0.4}$ & $30.9 \mtiny{\pm 0.2}$ \\
\bottomrule
\end{tabular}
\endgroup
\end{small}
\end{center}
\vskip -0.1in
\end{table}

\subsection{Extra ACC and BWT results}\label{sec_appx:extra_bwt_results}
\begin{table*}
\caption{\emph{Average accuracy} (ACC) and \emph{backward transfer} (BWT) scores (\%) for LLaVA with the fine-tuned CLIP-L-14. Each column indicates the setting and fine-tuning method. We include error bars over 3 runs.}
\label{table:bwt_metrics_clip_full}
\vskip 0.15in
\begin{center}
\begingroup
\setlength{\tabcolsep}{5.6pt}
\begin{tabular}{l c c c c c c c c c}
\toprule
 & & & & \multicolumn{6}{c}{\textbf{FT Method}}  \\
\cmidrule(lr){5-10}
\multirow{2}{*}{\textbf{Setting}} & \multirow{2}{*}{\textbf{FT Dataset}}  & \multicolumn{2}{c}{\textbf{Zr-Shot}} & \multicolumn{2}{c}{\textbf{LoRA}} & \multicolumn{2}{c}{\textbf{SPU}} &  \multicolumn{2}{c}{\textbf{LoRSU}} \\
\cmidrule(lr){3-4} \cmidrule(lr){5-6} \cmidrule(lr){7-8} \cmidrule(lr){9-10} & & \textbf{ACC ($\uparrow)$} & \textbf{BWT ($\uparrow)$} & \textbf{ACC ($\uparrow)$} & \textbf{BWT ($\uparrow)$} & \textbf{ACC ($\uparrow)$} & \textbf{BWT ($\uparrow)$} & \textbf{ACC ($\uparrow)$} & \textbf{BWT ($\uparrow)$} \\
\midrule
\multirow{4}{*}{\textbf{CL-5}} & \textbf{GTS} & $75.4$ & $0.0$ & $79.2 \mtiny{\pm 0.7}$ & $-7.1 \mtiny{\pm 0.8}$ & $80.8 \mtiny{\pm 0.5}$ & $0.5 \mtiny{\pm 0.6}$ & $81.1 \mtiny{\pm 0.6}$ & $0.4 \mtiny{\pm 0.7}$ \\
& \textbf{TSI} & $54.0$ & $0.0$ & $55.5 \mtiny{\pm 0.9}$ & $-2.5 \mtiny{\pm 0.6}$ & $55.5 \mtiny{\pm 0.6}$ & $0.2 \mtiny{\pm 0.5}$ & $57.0 \mtiny{\pm 0.8}$ & $0.5 \mtiny{\pm 0.6}$ \\
& \textbf{AIR} & $60.4$ & $0.0$ & $59.2 \mtiny{\pm 0.8}$ & $-2.1 \mtiny{\pm 0.7}$ & $64.7 \mtiny{\pm 0.5}$ & $2.8 \mtiny{\pm 0.6}$ & $65.0 \mtiny{\pm 0.7}$ & $2.5 \mtiny{\pm 0.6}$ \\
& \textbf{ESAT} & $76.4$ & $0.0$ & $73.8 \mtiny{\pm 0.9}$ & $-3.4 \mtiny{\pm 0.6}$ & $79.8 \mtiny{\pm 0.6}$ & $1.5 \mtiny{\pm 0.7}$ & $82.2 \mtiny{\pm 0.7}$ & $2.0 \mtiny{\pm 0.6}$ \\
\midrule
\multirow{4}{*}{\textbf{CL-20}} & \textbf{GTS} & $75.4$ & $0.0$ & $77.2 \mtiny{\pm 0.4}$ & $-9.1 \mtiny{\pm 0.5}$ & $82.8 \mtiny{\pm 0.4}$ & $-0.6 \mtiny{\pm 0.3}$ & $83.5 \mtiny{\pm 0.6}$ & $-0.4 \mtiny{\pm 0.3}$ \\
\bottomrule
\end{tabular}
\endgroup
\end{center}
\vskip -0.1in
\end{table*}

In Table~\ref{table:bwt_metrics_clip_full} we present results of the ACC and BWT on extra datasets plus the ones in the main paper. The results follow the same patterns as in section~\ref{sec:experiments} with LoRSU demonstrating the most consistent performance in both ACC and BWT compared to the other two baselines. SPU is close to \ours in terms of BWT but it significantly lacks behind in ACC.

\subsection{CLIP-based vs. Perplexity-based Updates+}
The detailed accuracies for all baselines and datasets used to create Table~\ref{table:ppl_vs_clip_summary} of the main paper can be found in Tables~\ref{table:fine_tune_llm_gtsrb} through~\ref{table:fine_tune_llm_eurosat}. We have also included results on fine-tuning the model using \emph{MMVP} dataset in Table~\ref{table:fine_tune_llm_mmvp}.
\begin{table}
\caption{Exact accuracy scores (\%) for each baseline used to fine-tune the model on the \emph{GTS} dataset under three different continual learning (5, 10, 50 shots)  settings. We include error bars over 3 runs.}
 \label{table:fine_tune_llm_gtsrb}
\vskip 0.15in
\begin{center}
\begin{small}
\begingroup
\setlength{\tabcolsep}{3.6pt}
\begin{tabular}{l c c c c c c c c c c c}
\toprule
 & & \multicolumn{10}{c}{\textbf{VQA Datasets (Acc \%)}}  \\
\cmidrule(lr){3-12}
\textbf{Setting} & \textbf{PEFT Method}  & \textbf{GTS} & \textbf{TSI} & \textbf{CAn} & \textbf{AIR} & \textbf{ESAT} & \textbf{DALLE} & \textbf{VSR} & \textbf{HM} & \textbf{MMVP} & \textbf{VisOnly} \\
\midrule
 & \textbf{Zr-Shot} & $75.6$ & $53.1$ & $82.7$ & $60.4$ & $76.1$ & $91.1$ & $51.5$ & $61.2$ & $58.0$ & $31.3$ \\
\midrule
\multirow{6}{*}{\textbf{CL-5}} & \textbf{LoRA-L} & $71.5 \mtiny{\pm 1.2}$ & $52.3 \mtiny{\pm 0.5}$ & $81.2 \mtiny{\pm 0.6}$ & $60.0 \mtiny{\pm 1.2}$ & $75.5 \mtiny{\pm 0.9}$ & $91.5 \mtiny{\pm 1.3}$ & $51.9 \mtiny{\pm 1.5}$ & $61.2 \mtiny{\pm 1.1}$ & $57.6 \mtiny{\pm 0.3}$ & $32.2 \mtiny{\pm 0.5}$ \\
& \textbf{LoRA} & $76.3 \mtiny{\pm 0.8}$ & $52.6 \mtiny{\pm 1.4}$ & $73.3 \mtiny{\pm 0.6}$ & $56.7 \mtiny{\pm 1.2}$ & $49.3 \mtiny{\pm 0.8}$ & $87.1 \mtiny{\pm 1.3}$ & $51.8 \mtiny{\pm 1.2}$ & $61.3 \mtiny{\pm 1.2}$ & $58.1 \mtiny{\pm 0.3}$ & $31.6 \mtiny{\pm 0.4}$ \\
& \textbf{LoRSU} & $82.0 \mtiny{\pm 1.3}$ & $53.5 \mtiny{\pm 1.3}$ & $82.4 \mtiny{\pm 0.8}$ & $60.8 \mtiny{\pm 1.4}$ & $66.6 \mtiny{\pm 0.9}$ & $91.5 \mtiny{\pm 1.4}$ & $51.6 \mtiny{\pm 0.7}$ & $61.7 \mtiny{\pm 1.4}$ & $59.8 \mtiny{\pm 0.2}$ & $31.6 \mtiny{\pm 0.2}$ \\
& \textbf{LoRA-Ppl} & $68.1 \mtiny{\pm 0.8}$ & $54.5 \mtiny{\pm 1.4}$ & $80.7 \mtiny{\pm 0.6}$ & $59.3 \mtiny{\pm 1.2}$ & $52.8 \mtiny{\pm 0.8}$ & $90.7 \mtiny{\pm 1.3}$ & $51.7 \mtiny{\pm 1.2}$ & $60.7 \mtiny{\pm 1.2}$ & $54.8 \mtiny{\pm 0.4}$ & $33.4 \mtiny{\pm 0.5}$ \\
& \textbf{LoRA-F} & $72.9 \mtiny{\pm 0.9}$ & $54.0 \mtiny{\pm 0.7}$ & $81.5 \mtiny{\pm 0.9}$ & $59.6 \mtiny{\pm 0.8}$ & $61.9 \mtiny{\pm 0.8}$ & $90.3 \mtiny{\pm 1.1}$ & $51.9 \mtiny{\pm 0.8}$ & $60.9 \mtiny{\pm 1.2}$ & $58.4 \mtiny{\pm 0.4}$ & $31.1 \mtiny{\pm 0.3}$ \\
\bottomrule
\end{tabular}
\endgroup
\end{small}
\end{center}
\vskip -0.1in
\end{table}

\begin{table}
\caption{Exact accuracy scores (\%) for each baseline used to fine-tune the model on the \emph{TSI} dataset under three different continual learning (5, 10, 50 shots)  settings. We include error bars over 3 runs.}
 \label{table:fine_tune_llm_tsi}
\vskip 0.15in
\begin{center}
\begin{small}
\begingroup
\setlength{\tabcolsep}{3.6pt}
\begin{tabular}{l c c c c c c c c c c c}
\toprule
 & & \multicolumn{10}{c}{\textbf{VQA Datasets (Acc \%)}}  \\
\cmidrule(lr){3-12}
\textbf{Setting} & \textbf{PEFT Method}  & \textbf{GTS} & \textbf{TSI} & \textbf{CAn} & \textbf{AIR} & \textbf{ESAT} & \textbf{DALLE} & \textbf{VSR} & \textbf{HM} & \textbf{MMVP} & \textbf{VisOnly} \\
\midrule
 & \textbf{Zr-Shot} & $75.6$ & $53.1$ & $82.7$ & $60.4$ & $76.1$ & $91.1$ & $51.5$ & $61.2$ & $58.0$ & $31.3$ \\
\midrule
\multirow{6}{*}{\textbf{CL-5}} & \textbf{LoRA-L} & $76.0 \mtiny{\pm 1.5}$ & $59.1 \mtiny{\pm 0.6}$ & $82.7 \mtiny{\pm 0.9}$ & $60.7 \mtiny{\pm 0.7}$ & $75.9 \mtiny{\pm 0.9}$ & $91.5 \mtiny{\pm 1.0}$ & $51.5 \mtiny{\pm 0.9}$ & $63.6 \mtiny{\pm 1.2}$ & $54.1 \mtiny{\pm 0.4}$ & $31.2 \mtiny{\pm 0.4}$ \\
& \textbf{LoRA} & $73.4 \mtiny{\pm 1.0}$ & $53.0 \mtiny{\pm 0.9}$ & $80.2 \mtiny{\pm 0.6}$ & $58.8 \mtiny{\pm 0.7}$ & $59.1 \mtiny{\pm 1.4}$ & $90.2 \mtiny{\pm 1.1}$ & $51.6 \mtiny{\pm 1.3}$ & $61.2 \mtiny{\pm 1.4}$ & $56.7 \mtiny{\pm 0.4}$ & $31.7 \mtiny{\pm 0.4}$ \\
& \textbf{LoRSU} & $75.9 \mtiny{\pm 0.9}$ & $56.3 \mtiny{\pm 0.7}$ & $82.7 \mtiny{\pm 0.9}$ & $60.8 \mtiny{\pm 1.0}$ & $76.2 \mtiny{\pm 1.4}$ & $91.3 \mtiny{\pm 1.2}$ & $51.6 \mtiny{\pm 0.9}$ & $61.7 \mtiny{\pm 0.8}$ & $57.7 \mtiny{\pm 0.3}$ & $31.2 \mtiny{\pm 0.3}$ \\
& \textbf{LoRA-Ppl} & $75.0 \mtiny{\pm 1.0}$ & $64.0 \mtiny{\pm 0.6}$ & $82.8 \mtiny{\pm 1.3}$ & $58.4 \mtiny{\pm 1.0}$ & $60.8 \mtiny{\pm 0.8}$ & $88.7 \mtiny{\pm 1.3}$ & $51.6 \mtiny{\pm 1.4}$ & $61.5 \mtiny{\pm 1.0}$ & $55.0 \mtiny{\pm 0.4}$ & $32.2 \mtiny{\pm 0.4}$ \\
& \textbf{LoRA-F} & $75.3 \mtiny{\pm 0.5}$ & $45.1 \mtiny{\pm 1.1}$ & $82.5 \mtiny{\pm 0.9}$ & $57.2 \mtiny{\pm 1.5}$ & $73.2 \mtiny{\pm 1.0}$ & $83.9 \mtiny{\pm 1.2}$ & $53.8 \mtiny{\pm 0.9}$ & $64.3 \mtiny{\pm 1.3}$ & $45.6 \mtiny{\pm 0.3}$ & $30.9 \mtiny{\pm 0.4}$ \\
\bottomrule
\end{tabular}
\endgroup
\end{small}
\end{center}
\vskip -0.1in
\end{table}

\begin{table}
\caption{Exact accuracy scores (\%) for each baseline used to fine-tune the model on the \emph{CAn} dataset under three different continual learning (5, 10, 50 shots)  settings. We include error bars over 3 runs.}
 \label{table:fine_tune_llm_counteranimal}
\vskip 0.15in
\begin{center}
\begin{small}
\begingroup
\setlength{\tabcolsep}{3.6pt}
\begin{tabular}{l c c c c c c c c c c c}
\toprule
 & & \multicolumn{10}{c}{\textbf{VQA Datasets (Acc \%)}}  \\
\cmidrule(lr){3-12}
\textbf{Setting} & \textbf{PEFT Method}  & \textbf{GTS} & \textbf{TSI} & \textbf{CAn} & \textbf{AIR} & \textbf{ESAT} & \textbf{DALLE} & \textbf{VSR} & \textbf{HM} & \textbf{MMVP} & \textbf{VisOnly} \\
\midrule
 & \textbf{Zr-Shot} & $75.6$ & $53.1$ & $82.7$ & $60.4$ & $76.1$ & $91.1$ & $51.5$ & $61.2$ & $58.0$ & $31.3$ \\
\midrule
\multirow{6}{*}{\textbf{CL-5}} & \textbf{LoRA-L} & $75.5 \mtiny{\pm 1.4}$ & $53.1 \mtiny{\pm 0.8}$ & $79.4 \mtiny{\pm 1.4}$ & $59.2 \mtiny{\pm 0.6}$ & $75.2 \mtiny{\pm 0.9}$ & $91.5 \mtiny{\pm 1.1}$ & $52.4 \mtiny{\pm 1.3}$ & $60.2 \mtiny{\pm 1.1}$ & $57.7 \mtiny{\pm 0.5}$ & $32.1 \mtiny{\pm 0.3}$ \\
& \textbf{LoRA} & $69.7 \mtiny{\pm 1.4}$ & $44.8 \mtiny{\pm 1.1}$ & $81.4 \mtiny{\pm 0.7}$ & $56.9 \mtiny{\pm 1.0}$ & $50.7 \mtiny{\pm 1.3}$ & $92.9 \mtiny{\pm 1.3}$ & $52.0 \mtiny{\pm 1.0}$ & $61.8 \mtiny{\pm 1.5}$ & $56.5 \mtiny{\pm 0.4}$ & $31.3 \mtiny{\pm 0.4}$ \\
& \textbf{LoRSU} & $75.2 \mtiny{\pm 0.8}$ & $52.7 \mtiny{\pm 0.9}$ & $83.0 \mtiny{\pm 1.0}$ & $60.1 \mtiny{\pm 0.7}$ & $76.8 \mtiny{\pm 1.0}$ & $91.8 \mtiny{\pm 1.4}$ & $51.6 \mtiny{\pm 1.1}$ & $62.3 \mtiny{\pm 1.2}$ & $58.7 \mtiny{\pm 0.3}$ & $31.4 \mtiny{\pm 0.4}$ \\
& \textbf{LoRA-Ppl} & $65.8 \mtiny{\pm 1.1}$ & $50.7 \mtiny{\pm 0.6}$ & $79.2 \mtiny{\pm 0.5}$ & $48.4 \mtiny{\pm 1.4}$ & $63.0 \mtiny{\pm 1.2}$ & $86.7 \mtiny{\pm 1.3}$ & $51.8 \mtiny{\pm 1.0}$ & $57.2 \mtiny{\pm 1.4}$ & $52.5 \mtiny{\pm 0.3}$ & $32.4 \mtiny{\pm 0.4}$ \\
& \textbf{LoRA-F} & $70.1 \mtiny{\pm 0.6}$ & $52.2 \mtiny{\pm 0.7}$ & $78.6 \mtiny{\pm 0.7}$ & $50.9 \mtiny{\pm 0.9}$ & $73.4 \mtiny{\pm 0.8}$ & $91.3 \mtiny{\pm 1.0}$ & $54.7 \mtiny{\pm 0.8}$ & $62.2 \mtiny{\pm 1.4}$ & $58.0 \mtiny{\pm 0.5}$ & $31.3 \mtiny{\pm 0.3}$ \\
\bottomrule
\end{tabular}
\endgroup
\end{small}
\end{center}
\vskip -0.1in
\end{table}

\begin{table}
\caption{Exact accuracy scores (\%) for each baseline used to fine-tune the model on the \emph{AIR} dataset under three different continual learning (5, 10, 50 shots)  settings. We include error bars over 3 runs.}
 \label{table:fine_tune_llm_aircraft}
 \vskip 0.15in
\begin{center}
\begin{small}
\begingroup
\setlength{\tabcolsep}{3.6pt}
\begin{tabular}{l c c c c c c c c c c c}
\toprule
 & & \multicolumn{10}{c}{\textbf{VQA Datasets (Acc \%)}}  \\
\cmidrule(lr){3-12}
\textbf{Setting} & \textbf{PEFT Method}  & \textbf{GTS} & \textbf{TSI} & \textbf{CAn} & \textbf{AIR} & \textbf{ESAT} & \textbf{DALLE} & \textbf{VSR} & \textbf{HM} & \textbf{MMVP} & \textbf{VisOnly} \\
\midrule
 & \textbf{Zr-Shot} & $75.6$ & $53.1$ & $82.7$ & $60.4$ & $76.1$ & $91.1$ & $51.5$ & $61.2$ & $58.0$ & $31.3$ \\
\midrule
\multirow{6}{*}{\textbf{CL-5}} & \textbf{LoRA-L} & $75.6 \mtiny{\pm 0.7}$ & $54.4 \mtiny{\pm 0.5}$ & $81.8 \mtiny{\pm 1.1}$ & $58.7 \mtiny{\pm 0.9}$ & $75.7 \mtiny{\pm 1.4}$ & $92.0 \mtiny{\pm 1.4}$ & $51.6 \mtiny{\pm 0.9}$ & $61.0 \mtiny{\pm 0.6}$ & $59.1 \mtiny{\pm 0.3}$ & $32.2 \mtiny{\pm 0.5}$ \\
& \textbf{LoRA} & $70.9 \mtiny{\pm 0.9}$ & $52.7 \mtiny{\pm 0.6}$ & $79.0 \mtiny{\pm 0.7}$ & $61.7 \mtiny{\pm 0.5}$ & $48.8 \mtiny{\pm 0.7}$ & $90.6 \mtiny{\pm 0.6}$ & $52.0 \mtiny{\pm 0.9}$ & $62.5 \mtiny{\pm 0.8}$ & $60.0 \mtiny{\pm 0.3}$ & $31.1 \mtiny{\pm 0.2}$ \\
& \textbf{LoRSU} & $76.2 \mtiny{\pm 0.8}$ & $53.4 \mtiny{\pm 1.4}$ & $82.5 \mtiny{\pm 1.0}$ & $65.2 \mtiny{\pm 1.3}$ & $76.0 \mtiny{\pm 0.9}$ & $91.8 \mtiny{\pm 0.8}$ & $51.6 \mtiny{\pm 0.8}$ & $62.1 \mtiny{\pm 1.1}$ & $59.0 \mtiny{\pm 0.4}$ & $31.2 \mtiny{\pm 0.3}$ \\
& \textbf{LoRA-Ppl} & $74.9 \mtiny{\pm 0.8}$ & $54.2 \mtiny{\pm 1.2}$ & $79.1 \mtiny{\pm 0.5}$ & $59.7 \mtiny{\pm 0.9}$ & $68.5 \mtiny{\pm 0.9}$ & $90.8 \mtiny{\pm 1.3}$ & $51.8 \mtiny{\pm 0.7}$ & $62.0 \mtiny{\pm 0.7}$ & $55.1 \mtiny{\pm 0.5}$ & $31.1 \mtiny{\pm 0.5}$ \\
& \textbf{LoRA-F} & $72.3 \mtiny{\pm 0.5}$ & $50.6 \mtiny{\pm 1.3}$ & $78.7 \mtiny{\pm 1.4}$ & $70.0 \mtiny{\pm 1.3}$ & $64.4 \mtiny{\pm 0.9}$ & $90.9 \mtiny{\pm 0.6}$ & $54.9 \mtiny{\pm 1.3}$ & $57.7 \mtiny{\pm 1.1}$ & $62.0 \mtiny{\pm 0.6}$ & $32.2 \mtiny{\pm 0.5}$ \\
\bottomrule
\end{tabular}
\endgroup
\end{small}
\end{center}
\vskip -0.1in
\end{table}

\begin{table}
\caption{Exact accuracy scores (\%) for each baseline used to fine-tune the model on the \emph{ESAT} dataset under three different continual learning (5, 10, 50 shots)  settings. We include error bars over 3 runs.}
 \label{table:fine_tune_llm_eurosat}
 \vskip 0.15in
\begin{center}
\begin{small}
\begingroup
\setlength{\tabcolsep}{3.6pt}
\begin{tabular}{l c c c c c c c c c c c}
\toprule
 & & \multicolumn{10}{c}{\textbf{VQA Datasets (Acc \%)}}  \\
\cmidrule(lr){3-12}
\textbf{Setting} & \textbf{PEFT Method}  & \textbf{GTS} & \textbf{TSI} & \textbf{CAn} & \textbf{AIR} & \textbf{ESAT} & \textbf{DALLE} & \textbf{VSR} & \textbf{HM} & \textbf{MMVP} & \textbf{VisOnly} \\
\midrule
 & \textbf{Zr-Shot} & $75.6$ & $53.1$ & $82.7$ & $60.4$ & $76.1$ & $91.1$ & $51.5$ & $61.2$ & $58.0$ & $31.3$ \\
\midrule
\multirow{6}{*}{\textbf{CL-5}} & \textbf{LoRA-L} & $75.4 \mtiny{\pm 0.7}$ & $52.2 \mtiny{\pm 1.4}$ & $82.8 \mtiny{\pm 0.6}$ & $60.6 \mtiny{\pm 1.5}$ & $75.9 \mtiny{\pm 1.1}$ & $91.7 \mtiny{\pm 0.9}$ & $51.5 \mtiny{\pm 0.6}$ & $60.2 \mtiny{\pm 0.8}$ & $57.6 \mtiny{\pm 0.5}$ & $31.6 \mtiny{\pm 0.4}$ \\
& \textbf{LoRA} & $73.2 \mtiny{\pm 1.3}$ & $49.3 \mtiny{\pm 1.2}$ & $80.6 \mtiny{\pm 0.9}$ & $60.4 \mtiny{\pm 1.1}$ & $74.5 \mtiny{\pm 0.8}$ & $92.3 \mtiny{\pm 1.3}$ & $52.0 \mtiny{\pm 1.1}$ & $61.6 \mtiny{\pm 1.1}$ & $57.4 \mtiny{\pm 0.4}$ & $31.4 \mtiny{\pm 0.3}$ \\
& \textbf{LoRSU} & $76.2 \mtiny{\pm 1.0}$ & $53.6 \mtiny{\pm 1.1}$ & $82.5 \mtiny{\pm 1.2}$ & $60.8 \mtiny{\pm 0.8}$ & $82.9 \mtiny{\pm 1.0}$ & $91.5 \mtiny{\pm 0.9}$ & $51.6 \mtiny{\pm 0.9}$ & $61.3 \mtiny{\pm 0.7}$ & $57.7 \mtiny{\pm 0.4}$ & $31.9 \mtiny{\pm 0.4}$ \\
& \textbf{LoRA-Ppl} & $76.0 \mtiny{\pm 0.7}$ & $52.6 \mtiny{\pm 1.0}$ & $82.6 \mtiny{\pm 1.3}$ & $60.4 \mtiny{\pm 1.4}$ & $75.5 \mtiny{\pm 0.9}$ & $91.9 \mtiny{\pm 1.0}$ & $51.8 \mtiny{\pm 0.9}$ & $62.8 \mtiny{\pm 0.8}$ & $59.0 \mtiny{\pm 0.4}$ & $31.6 \mtiny{\pm 0.5}$ \\
& \textbf{LoRA-F} & $74.3 \mtiny{\pm 1.3}$ & $51.5 \mtiny{\pm 1.4}$ & $81.1 \mtiny{\pm 1.0}$ & $60.3 \mtiny{\pm 1.1}$ & $81.5 \mtiny{\pm 1.2}$ & $90.8 \mtiny{\pm 1.2}$ & $51.9 \mtiny{\pm 1.2}$ & $61.9 \mtiny{\pm 1.2}$ & $57.7 \mtiny{\pm 0.2}$ & $31.3 \mtiny{\pm 0.5}$ \\
\bottomrule
\end{tabular}
\endgroup
\end{small}
\end{center}
\vskip -0.1in
\end{table}

\begin{table}
\caption{Exact accuracy scores (\%) for each baseline used to fine-tune the model on the \emph{HM} dataset under three different continual learning (5, 10, 50 shots) settings. We include error bars over 3 runs.}
 \label{table:fine_tune_llm_hm}
 \vskip 0.15in
\begin{center}
\begin{small}
\begingroup
\setlength{\tabcolsep}{3.6pt}
\begin{tabular}{l c c c c c c c c c c c}
\toprule
 & & \multicolumn{10}{c}{\textbf{VQA Datasets (Acc \%)}}  \\
\cmidrule(lr){3-12}
\textbf{Setting} & \textbf{PEFT Method}  & \textbf{GTS} & \textbf{TSI} & \textbf{CAn} & \textbf{AIR} & \textbf{ESAT} & \textbf{DALLE} & \textbf{VSR} & \textbf{HM} & \textbf{MMVP} & \textbf{VisOnly} \\
\midrule
 & \textbf{Zr-Shot} & $75.6$ & $53.1$ & $82.7$ & $60.4$ & $76.1$ & $91.1$ & $51.5$ & $61.2$ & $58.0$ & $31.3$ \\
\midrule
\multirow{6}{*}{\textbf{CL-5}} & \textbf{LoRA-L} & $76.5 \mtiny{\pm 1.0}$ & $51.5 \mtiny{\pm 1.1}$ & $83.2 \mtiny{\pm 1.2}$ & $60.5 \mtiny{\pm 0.8}$ & $75.7 \mtiny{\pm 1.0}$ & $90.9 \mtiny{\pm 0.9}$ & $51.6 \mtiny{\pm 0.9}$ & $68.6 \mtiny{\pm 0.7}$ & $34.4 \mtiny{\pm 0.5}$ & $31.1 \mtiny{\pm 0.5}$ \\
& \textbf{LoRA} & $68.8 \mtiny{\pm 0.8}$ & $47.0 \mtiny{\pm 1.0}$ & $70.5 \mtiny{\pm 0.8}$ & $51.7 \mtiny{\pm 1.1}$ & $54.1 \mtiny{\pm 0.6}$ & $89.1 \mtiny{\pm 0.8}$ & $52.2 \mtiny{\pm 1.5}$ & $60.8 \mtiny{\pm 0.8}$ & $54.7 \mtiny{\pm 0.6}$ & $30.5 \mtiny{\pm 0.3}$ \\
& \textbf{LoRSU} & $75.7 \mtiny{\pm 1.2}$ & $54.1 \mtiny{\pm 1.1}$ & $82.9 \mtiny{\pm 0.6}$ & $60.7 \mtiny{\pm 1.0}$ & $76.3 \mtiny{\pm 1.1}$ & $92.2 \mtiny{\pm 0.6}$ & $51.5 \mtiny{\pm 0.9}$ & $61.8 \mtiny{\pm 1.2}$ & $58.1 \mtiny{\pm 0.2}$ & $31.9 \mtiny{\pm 0.5}$ \\
& \textbf{LoRA-Ppl} & $76.2 \mtiny{\pm 0.6}$ & $48.4 \mtiny{\pm 1.4}$ & $82.5 \mtiny{\pm 1.2}$ & $57.2 \mtiny{\pm 0.9}$ & $72.8 \mtiny{\pm 0.9}$ & $90.9 \mtiny{\pm 0.9}$ & $51.8 \mtiny{\pm 1.0}$ & $60.0 \mtiny{\pm 1.0}$ & $56.4 \mtiny{\pm 0.4}$ & $33.1 \mtiny{\pm 0.4}$ \\
& \textbf{LoRA-F} & $71.8 \mtiny{\pm 1.1}$ & $47.8 \mtiny{\pm 0.8}$ & $79.9 \mtiny{\pm 1.5}$ & $57.6 \mtiny{\pm 1.0}$ & $63.2 \mtiny{\pm 1.1}$ & $90.1 \mtiny{\pm 1.0}$ & $48.0 \mtiny{\pm 0.7}$ & $67.2 \mtiny{\pm 0.9}$ & $49.0 \mtiny{\pm 0.3}$ & $31.5 \mtiny{\pm 0.2}$ \\
\bottomrule
\end{tabular}
\endgroup
\end{small}
\end{center}
\vskip -0.1in
\end{table}

\begin{table}
\caption{Exact accuracy scores (\%) for each baseline used to fine-tune the model on the \emph{VSR} dataset under three different continual learning (5, 10, 50 shots)  settings. We include error bars over 3 runs.}
 \label{table:fine_tune_llm_vsr}
 \vskip 0.15in
\begin{center}
\begin{small}
\begingroup
\setlength{\tabcolsep}{3.6pt}
\begin{tabular}{l c c c c c c c c c c c}
\toprule
 & & \multicolumn{10}{c}{\textbf{VQA Datasets (Acc \%)}}  \\
\cmidrule(lr){3-12}
\textbf{Setting} & \textbf{PEFT Method}  & \textbf{GTS} & \textbf{TSI} & \textbf{CAn} & \textbf{AIR} & \textbf{ESAT} & \textbf{DALLE} & \textbf{VSR} & \textbf{HM} & \textbf{MMVP} & \textbf{VisOnly} \\
\midrule
 & \textbf{Zr-Shot} & $75.6$ & $53.1$ & $82.7$ & $60.4$ & $76.1$ & $91.1$ & $51.5$ & $61.2$ & $58.0$ & $31.3$ \\
\midrule
\multirow{6}{*}{\textbf{CL-5}} & \textbf{LoRA-L} & $75.3 \mtiny{\pm 0.7}$ & $59.9 \mtiny{\pm 1.4}$ & $81.0 \mtiny{\pm 1.1}$ & $56.2 \mtiny{\pm 0.5}$ & $66.8 \mtiny{\pm 1.3}$ & $90.1 \mtiny{\pm 1.3}$ & $68.3 \mtiny{\pm 1.1}$ & $65.0 \mtiny{\pm 1.4}$ & $57.6 \mtiny{\pm 0.3}$ & $32.5 \mtiny{\pm 0.4}$ \\
& \textbf{LoRA} & $72.6 \mtiny{\pm 1.3}$ & $49.5 \mtiny{\pm 1.5}$ & $78.2 \mtiny{\pm 0.8}$ & $57.5 \mtiny{\pm 1.5}$ & $55.0 \mtiny{\pm 0.9}$ & $88.8 \mtiny{\pm 0.7}$ & $52.0 \mtiny{\pm 1.0}$ & $61.9 \mtiny{\pm 1.5}$ & $59.7 \mtiny{\pm 0.3}$ & $30.4 \mtiny{\pm 0.5}$ \\
& \textbf{LoRSU} & $75.6 \mtiny{\pm 0.7}$ & $52.2 \mtiny{\pm 1.4}$ & $82.2 \mtiny{\pm 0.6}$ & $60.1 \mtiny{\pm 0.9}$ & $77.9 \mtiny{\pm 0.6}$ & $91.1 \mtiny{\pm 1.1}$ & $51.9 \mtiny{\pm 1.3}$ & $62.2 \mtiny{\pm 1.5}$ & $58.4 \mtiny{\pm 0.3}$ & $31.7 \mtiny{\pm 0.3}$ \\
& \textbf{LoRA-Ppl} & $65.8 \mtiny{\pm 0.7}$ & $48.7 \mtiny{\pm 0.8}$ & $65.4 \mtiny{\pm 1.3}$ & $33.8 \mtiny{\pm 1.4}$ & $48.8 \mtiny{\pm 0.5}$ & $81.7 \mtiny{\pm 1.2}$ & $61.7 \mtiny{\pm 0.5}$ & $56.2 \mtiny{\pm 0.7}$ & $43.6 \mtiny{\pm 0.2}$ & $32.8 \mtiny{\pm 0.4}$ \\
& \textbf{LoRA-F} & $76.0 \mtiny{\pm 0.9}$ & $64.5 \mtiny{\pm 0.8}$ & $81.2 \mtiny{\pm 1.3}$ & $57.6 \mtiny{\pm 0.6}$ & $69.7 \mtiny{\pm 1.5}$ & $89.4 \mtiny{\pm 0.8}$ & $69.5 \mtiny{\pm 1.0}$ & $12.8 \mtiny{\pm 0.5}$ & $30.3 \mtiny{\pm 0.5}$ & $13.0 \mtiny{\pm 0.3}$ \\
\bottomrule
\end{tabular}
\endgroup
\end{small}
\end{center}
\vskip -0.1in
\end{table}

\begin{table}
\caption{Exact accuracy scores (\%) for each baseline used to fine-tune the model on the \emph{MMVP} dataset under three different continual learning (5, 10, 50 shots)  settings. We include error bars over 3 runs.}
 \label{table:fine_tune_llm_mmvp}
 \vskip 0.15in
\begin{center}
\begingroup
\setlength{\tabcolsep}{6.4pt}
\begin{tabular}{l c c c c c c c c c c c}
\toprule
 & & \multicolumn{10}{c}{\textbf{VQA Datasets (Acc \%)}}  \\
\cmidrule(lr){3-12}
\textbf{Setting} & \textbf{PEFT Method}  & \textbf{GTS} & \textbf{TSI} & \textbf{CAn} & \textbf{AIR} & \textbf{ESAT} & \textbf{DALLE} & \textbf{VSR} & \textbf{HM} & \textbf{MMVP} & \textbf{VisOnly} \\
\midrule
 & \textbf{Zr-Shot} & $75.6$ & $53.1$ & $82.7$ & $60.4$ & $76.1$ & $91.1$ & $51.5$ & $61.2$ & $58.0$ & $31.3$ \\
\midrule
\multirow{4}{*}{\textbf{CL}} & \textbf{LoRA-L} & $75.5$ & $52.8$ & $82.0$ & $60.5$ & $76.0$ & $91.5$ & $51.5$ & $63.6$ & $57.7$ & $30.6$ \\
& \textbf{LoRA-Ppl} & $75.5$ & $53.6$ & $83.0$ & $60.3$ & $75.6$ & $91.1$ & $51.5$ & $63.1$ & $60.7$ & $31.7$ \\
& \textbf{LoRA-F} & $75.2$ & $52.9$ & $81.3$ & $60.5$ & $74.3$ & $90.4$ & $51.6$ & $63.6$ & $60.0$ & $31.4$ \\
& \textbf{LoRSU-Ppl} & $75.1$ & $52.0$ & $81.2$ & $57.4$ & $75.2$ & $90.2$ & $51.7$ & $63.9$ & $60.3$ & $30.8$ \\
\bottomrule
\end{tabular}
\endgroup
\end{center}
\vskip -0.1in
\end{table}

\begin{table}
\caption{Exact accuracy scores (\%) for each baseline used to fine-tune the model on the \emph{VisOnly} dataset under three different continual learning (5, 10, 50 shots)  settings. We include error bars over 3 runs.}
 \label{table:fine_tune_llm_visonlyqa}
 \vskip 0.15in
\begin{center}
\begin{small}
\begingroup
\setlength{\tabcolsep}{3.6pt}
\begin{tabular}{l c c c c c c c c c c c}
\toprule
 & & \multicolumn{10}{c}{\textbf{VQA Datasets (Acc \%)}}  \\
\cmidrule(lr){3-12}
\textbf{Setting} & \textbf{PEFT Method}  & \textbf{GTS} & \textbf{TSI} & \textbf{CAn} & \textbf{AIR} & \textbf{ESAT} & \textbf{DALLE} & \textbf{VSR} & \textbf{HM} & \textbf{MMVP} & \textbf{VisOnly} \\
\midrule
 & \textbf{Zr-Shot} & $75.6$ & $53.1$ & $82.7$ & $60.4$ & $76.1$ & $91.1$ & $51.5$ & $61.2$ & $58.0$ & $31.3$ \\
\midrule
\multirow{6}{*}{\textbf{CL-5}} & \textbf{LoRA-L} & $76.5 \mtiny{\pm 1.2}$ & $51.9 \mtiny{\pm 0.7}$ & $82.4 \mtiny{\pm 1.4}$ & $60.5 \mtiny{\pm 1.5}$ & $76.1 \mtiny{\pm 1.0}$ & $91.5 \mtiny{\pm 0.6}$ & $51.6 \mtiny{\pm 0.9}$ & $60.3 \mtiny{\pm 1.0}$ & $57.6 \mtiny{\pm 0.3}$ & $31.3 \mtiny{\pm 0.4}$ \\
& \textbf{LoRA} & $70.9 \mtiny{\pm 1.4}$ & $52.1 \mtiny{\pm 1.2}$ & $77.5 \mtiny{\pm 1.3}$ & $55.6 \mtiny{\pm 0.6}$ & $52.6 \mtiny{\pm 0.8}$ & $89.3 \mtiny{\pm 0.6}$ & $51.7 \mtiny{\pm 0.8}$ & $61.7 \mtiny{\pm 0.7}$ & $56.9 \mtiny{\pm 0.6}$ & $30.9 \mtiny{\pm 0.5}$ \\
& \textbf{LoRSU} & $75.9 \mtiny{\pm 0.7}$ & $53.1 \mtiny{\pm 0.8}$ & $82.5 \mtiny{\pm 0.7}$ & $60.4 \mtiny{\pm 1.0}$ & $76.1 \mtiny{\pm 1.5}$ & $91.9 \mtiny{\pm 0.8}$ & $51.5 \mtiny{\pm 1.3}$ & $61.3 \mtiny{\pm 1.2}$ & $58.9 \mtiny{\pm 0.4}$ & $31.5 \mtiny{\pm 0.2}$ \\
& \textbf{LoRA-Ppl} & $76.3 \mtiny{\pm 1.1}$ & $50.7 \mtiny{\pm 1.1}$ & $82.2 \mtiny{\pm 0.9}$ & $61.0 \mtiny{\pm 1.3}$ & $73.4 \mtiny{\pm 0.9}$ & $91.7 \mtiny{\pm 1.3}$ & $52.1 \mtiny{\pm 1.1}$ & $59.3 \mtiny{\pm 1.3}$ & $58.0 \mtiny{\pm 0.2}$ & $35.0 \mtiny{\pm 0.5}$ \\
& \textbf{LoRA-F} & $76.0 \mtiny{\pm 0.8}$ & $51.1 \mtiny{\pm 1.4}$ & $82.9 \mtiny{\pm 1.1}$ & $59.9 \mtiny{\pm 0.7}$ & $71.2 \mtiny{\pm 1.2}$ & $91.7 \mtiny{\pm 1.1}$ & $51.6 \mtiny{\pm 1.3}$ & $60.8 \mtiny{\pm 0.7}$ & $58.4 \mtiny{\pm 0.2}$ & $34.9 \mtiny{\pm 0.4}$ \\
\bottomrule
\end{tabular}
\endgroup
\end{small}
\end{center}
\vskip -0.1in
\end{table}

\section{Extra Ablation Studies}\label{sec_appx:extra_ablations}

\subsection{Ablation on the rank $r$ of LoRSU}
In Table~\ref{table:ablation_ranks}, we investigate the effect on performance of using different ranks for LoRSU. As the rank $r$ increases, the VQA accuracy on the target dataset slightly improves, peaking at $r=64$. Beyond 
that, performance slightly decreases. Performance on other datasets remains relatively stable with small fluctuations.
\begin{table}
\caption{Ablation study over the effect of the rank $r$ used by \emph{LoRSU} to fine-tune the image encoder, CLIP-L-14. We report the VQA accuracies of the last session in the \emph{50-shot} CL setting. The accuracies on the target dataset are in red color. For this experiment, we use two attention heads to fine-tune with LoRSU.}
 \label{table:ablation_ranks}
\vskip 0.15in
\begin{center}
\begingroup
\setlength{\tabcolsep}{6.7pt}
\begin{tabular}{l c c c c c c c c c c c}
\toprule
 & & \multicolumn{9}{c}{\textbf{VQA Datasets (Acc \%)}}  \\
\cmidrule(lr){3-12}
\textbf{FT Dataset} & \textbf{rank ($r$)}  & \textbf{GTS} & \textbf{TSI} & \textbf{CAn} & \textbf{AIR} & \textbf{ESAT} & \textbf{DALLE} & \textbf{VSR} & \textbf{HM} & \textbf{MMVP} & \textbf{VisOnly} \\
\midrule
\multirow{6}{*}{\textbf{GTS}} & \textbf{8} & \textcolor{red}{$83.0$} & $53.2$ & $81.3$ & $60.9$ & $61.0$ & $91.2$ & $51.5$ & $61.6$ & $60.0$ & $31.6$ \\
 & \textbf{16} & \textcolor{red}{$83.9$} & $53.4$ & $81.5$ & $60.2$ & $54.0$ & $91.4$ & $51.5$ & $62.1$ & $60.7$ & $31.6$ \\
 & \textbf{32} & \textcolor{red}{$84.8$} & $53.1$ & $81.9$ & $60.5$ & $58.0$ & $90.6$ & $51.6$ & $61.8$ & $58.7$ & $31.5$ \\
 & \textbf{64} & \textcolor{red}{$84.9$} &  $53.2$ & $81.3$ & $60.7$ & $61.7$ & $90.9$ & $51.5$ & $61.9$ & $59.3$ & $31.3$ \\
 & \textbf{128} & \textcolor{red}{$84.3$} & $53.2$ & $81.8$ & $60.6$ & $56.8$ & $91.5$ & $51.6$ & $61.8$ & $58.7$ & $31.2$ \\
 & \textbf{256} & \textcolor{red}{$84.5$} & $53.1$ & $81.5$ & $61.1$ & $51.5$ & $90.3$ & $51.6$ & $62.0$ & $58.7$ & $31.6$ \\
\midrule
\multirow{6}{*}{\textbf{TSI}} & \textbf{8} & $75.2$ & \textcolor{red}{$67.2$} & $82.0$ & $59.2$ & $71.6$ & $91.1$ & $51.5$ & $61.6$ & $58.0$ & $31.5$ \\
 & \textbf{16} & $75.4$ & \textcolor{red}{$68.0$} & $82.3$ & $59.1$ & $71.0$ & $90.6$ & $51.6$ & $61.6$ & $56.7$ & $31.2$ \\
 & \textbf{32} & $74.9$ & \textcolor{red}{$68.9$} & $81.8$ & $59.3$ & $70.1$ & $91.2$ & $51.5$ & $61.6$ & $58.0$ & $31.6$ \\
 & \textbf{64} & $75.3$ & \textcolor{red}{$72.1$} & $82.0$ & $59.3$ & $72.3$ & $90.5$ & $51.6$ & $61.4$ & $58.0$ & $31.6$ \\
 & \textbf{128} & $75.1$ & \textcolor{red}{$65.8$} & $81.7$ & $59.0$ & $70.0$ & $90.6$ & $51.5$ & $62.1$ & $56.7$ & $31.6$ \\
 & \textbf{256} & $75.4$ & \textcolor{red}{$66.4$} & $82.3$ & $59.6$ & $72.0$ & $91.2$ & $51.5$ & $62.1$ & $56.7$ & $31.5$ \\
\midrule
\midrule
\textbf{Zr-Shot} & & $75.6$ & $53.1$ & $82.7$ & $60.4$ & $76.1$ & $91.1$ & $51.5$ & $61.2$ & $58.0$ & $31.3$ \\
\bottomrule
\end{tabular}
\endgroup
\end{center}
\vskip -0.1in
\end{table}

\subsection{Ablation on the number of optimal attention heads of LoRSU}\label{sec_appx:ablation_heads_general}
\begin{table}
\caption{Ablation study over the effect of the number of attention heads used by \emph{LoRSU} to fine-tune the image encoder. We report the VQA accuracies of the last session in the \emph{50-shot} CL setting. The accuracies on the target dataset are in red color. For this experiment, we use $r=64$ for the rank of LoRSU.}
 \label{table:ablation_heads}
 \vskip 0.15in
\begin{center}
\begingroup
\setlength{\tabcolsep}{6.7pt}
\begin{tabular}{l c c c c c c c c c c c}
\toprule
 & & \multicolumn{9}{c}{\textbf{VQA Datasets (Acc \%)}}  \\
\cmidrule(lr){3-12}
\textbf{FT Dataset} & \textbf{\# heads}  & \textbf{GTS} & \textbf{TSI} & \textbf{CAn} & \textbf{AIR} & \textbf{ESAT} & \textbf{DALLE} & \textbf{VSR} & \textbf{HM} & \textbf{MMVP} & \textbf{VisOnly} \\
\midrule
\multirow{6}{*}{\textbf{GTS}} & \textbf{0} & \textcolor{red}{$83.1$} & $52.7$ & $82.2$ & $60.8$ & $60.6$ & $91.1$ & $51.6$ & $61.7$ & $59.3$ & $31.6$ \\
 & \textbf{1} &  \textcolor{red}{$83.9$} & $53.8$ & $82.0$ & $60.7$ & $55.4$ & $91.2$ & $51.6$ & $61.6$ & $60.0$ & $31.8$ \\
 & \textbf{2} &  \textcolor{red}{$84.9$} & $53.2$ & $81.3$ & $60.7$ & $61.7$ & $90.9$ & $51.5$ & $61.9$ & $59.3$ & $31.3$ \\
 & \textbf{4} &  \textcolor{red}{$84.7$} & $53.5$ & $81.0$ & $60.5$ & $60.5$ & $90.6$ & $51.5$ & $61.8$ & $58.7$ & $31.5$ \\
 & \textbf{8} &  \textcolor{red}{$84.9$} & $52.9$ & $81.2$ & $60.5$ & $58.8$ & $90.5$ & $51.5$ & $61.6$ & $59.3$ & $31.5$ \\
 & \textbf{16} &  \textcolor{red}{$85.0$} & $53.1$ & $81.3$ & $60.0$ & $59.2$ & $90.6$ & $51.5$ & $61.6$ & $56.7$ & $31.3$ \\
\midrule
\multirow{6}{*}{\textbf{TSI}} & \textbf{0} & $75.1$ &  \textcolor{red}{$64.2$} & $82.1$ & $59.3$ & $72.2$ & $90.8$ & $51.5$ & $61.8$ & $57.3$ & $31.5$ \\
 & \textbf{1} & $75.3$ &  \textcolor{red}{$64.8$} & $81.9$ & $59.5$ & $74.0$ & $90.5$ & $51.5$ & $61.6$ & $58.0$ & $32.0$ \\
 & \textbf{2} & $75.3$ &  \textcolor{red}{$72.1$} & $82.0$ & $59.3$ & $72.3$ & $90.5$ & $51.6$ & $61.4$ & $58.0$ & $31.6$ \\
 & \textbf{4} & $74.9$ &  \textcolor{red}{$66.8$} & $82.2$ & $58.9$ & $74.0$ & $90.5$ & $51.5$ & $62.1$ & $58.0$ & $31.4$ \\
 & \textbf{8} & $74.7$ &  \textcolor{red}{$67.4$} & $81.7$ & $59.1$ & $71.5$ & $91.2$ & $51.5$ & $62.2$ & $58.0$ & $31.7$ \\
 & \textbf{16} & $75.3$ &  \textcolor{red}{$65.2$} & $81.8$ & $59.9$ & $69.1$ & $90.5$ & $51.5$ & $61.6$ & $58.0$ & $31.3$ \\
\midrule
\midrule
\textbf{Zr-Shot} & & $75.6$ & $53.1$ & $82.7$ & $60.4$ & $76.1$ & $91.1$ & $51.5$ & $61.2$ & $58.0$ & $31.3$ \\
\bottomrule
\end{tabular}
\endgroup
\end{center}
\vskip -0.1in
\end{table}
In Table~\ref{table:ablation_heads}, we examine how the number of attention heads chosen to be fine-tuned affects LoRSU's performance. We notice that more attention heads marginally improve the performance of the model while the extra flexibility can cause more forgetting, e.g. ESAT.

\subsection{Robustness on the Choice of Attention Heads}\label{sec_appx:robustness}
\begin{table}
\caption{Robustness comparison of LoRSU with respect to the number of training epochs. We consider LoRSU, \emph{LoRSU-Rand} where the $k$ attention heads are chosen randomly and \emph{LoRSU-AAH} where all the attention heads are chosen for fine tuning. We use \emph{50 shots} on the \emph{GTS} for each method and we report the Target Improvement (\emph{TI}) on this dataset and the Control Change (\emph{CC}) using only ESAT as a control dataset.  We include error bars over 3 runs.}
 \label{table:lorsu_attn_rand_epochs}
\vskip 0.15in
\begin{center}
\begingroup
\setlength{\tabcolsep}{7.9pt}
\begin{tabular}{l c c c c c c}
\toprule
\multirow{2}{*}{\textbf{\# Epochs}} &  \multicolumn{2}{c}{\textbf{LoRSU-Rand}} & \multicolumn{2}{c}{\textbf{LoRSU-AAH}}  &  \multicolumn{2}{c}{\textbf{LoRSU}}  \\
 \cmidrule(lr){2-3}  \cmidrule(lr){4-5}  \cmidrule(lr){6-7} & \textbf{TI ($\uparrow)$} & \textbf{CC ($\uparrow)$} & \textbf{TI ($\uparrow)$} & \textbf{CC ($\uparrow)$} & \textbf{TI ($\uparrow)$} & \textbf{CC ($\uparrow)$} \\
\midrule
\textbf{2} & $5.2 \mtiny{\pm 0.9}$ & $-11.1 \mtiny{\pm 1.1}$ & $6.1 \mtiny{\pm 0.3}$   & $-11.6 \mtiny{\pm 0.7}$ & $5.6 \mtiny{\pm 0.4}$ & $-9.7 \mtiny{\pm 0.8}$  \\
\textbf{5} & $7.6 \mtiny{\pm 0.8}$ & $-15.0 \mtiny{\pm 0.9}$ & $9.3 \mtiny{\pm 0.4}$ & $-15.6 \mtiny{\pm 0.6}$ & $8.6 \mtiny{\pm 0.3}$ & $-12.6 \mtiny{\pm 0.5}$  \\
\textbf{10} & $7.8 \mtiny{\pm 0.5}$ & $-18.1 \mtiny{\pm 0.8}$ & $9.1 \mtiny{\pm 0.1}$ & $-19.6 \mtiny{\pm 0.5}$ & $9.7 \mtiny{\pm 0.1}$ & $-14.3 \mtiny{\pm 0.7}$  \\
\textbf{20} & $5.9 \mtiny{\pm 0.6}$ & $-20.0 \mtiny{\pm 0.7}$ & $8.1 \mtiny{\pm 0.1}$ & $-21.5 \mtiny{\pm 0.6}$ & $7.4 \mtiny{\pm 0.2}$ & $-15.7 \mtiny{\pm 0.6}$  \\
\bottomrule
\end{tabular}
\endgroup
\end{center}
\vskip -0.1in
\end{table}
We show in Table~\ref{table:lorsu_attn_rand_epochs} that LoRSU's mechanism of choosing the most important attention heads provides a clear advantage in terms of robustness over the other two LoRSU's variants, LoRSU-Rand and LoRSU-AAH. We can see that TI and CC decline in a lower rate compared to that of LoRSU-RAnd and LoRSU-AAH, as we increase the number of training epochs.. As expected, LoRSU-Rand appears to be the least robust method since the random choice of the attention heads constitute it more unstable.

\section{TSI vs. DALLE}\label{sec_appx:tsi_vs_dalle}
In Figures~\ref{fig:tsi_example_1} through~\ref{fig:tsi_example_4}, we present examples of images from TSI and DALLE for different actions. In general, we observe that TSI comprised of   natural, unposed images of senior individuals performing daily tasks, reflecting real-life scenarios. The images are broader, showing the surrounding environment, which is crucial for context. On the other hand, DALLE images are idealized or stylized images. The focus is narrower, with emphasis on the object of the action (e.g. tablet, glass, etc.).

\begin{figure}%
\vskip 0.2in
\centering
    \subfigure[TSI]{\includegraphics[width=.4\linewidth,height=5.2cm]{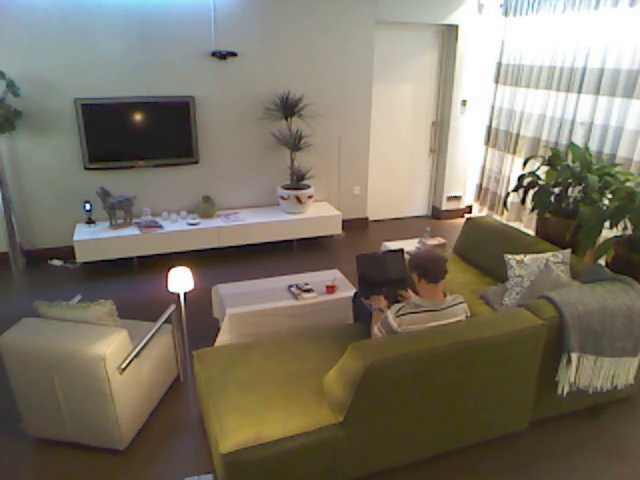}}
    \subfigure[DALLE]{\includegraphics[width=.4\linewidth,height=5.2cm]{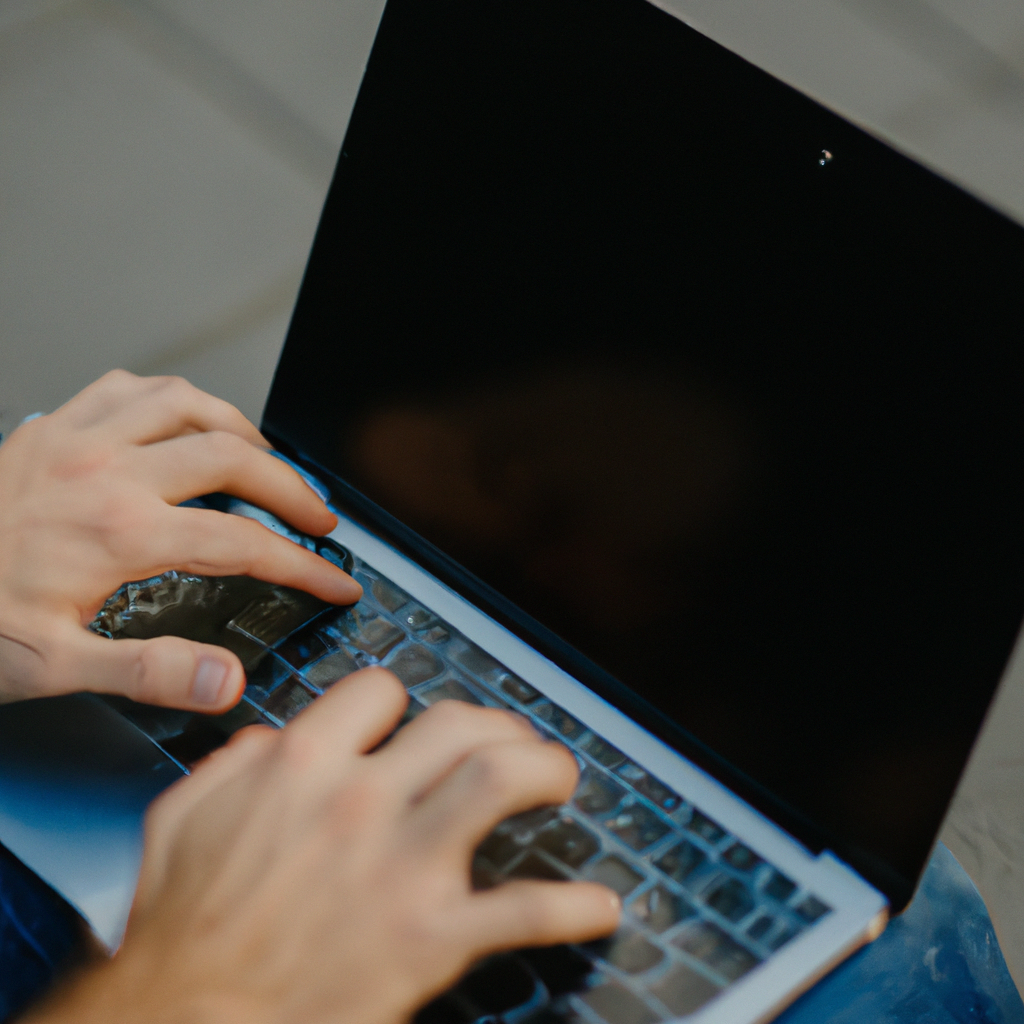}}
    \caption{Instances of the `Use Laptop' action.}%
    \label{fig:tsi_example_1}
\vskip -0.2in
\end{figure}

\begin{figure}%
\vskip 0.2in
\centering
    \subfigure[TSI]{\includegraphics[width=.4\linewidth,height=5.2cm]{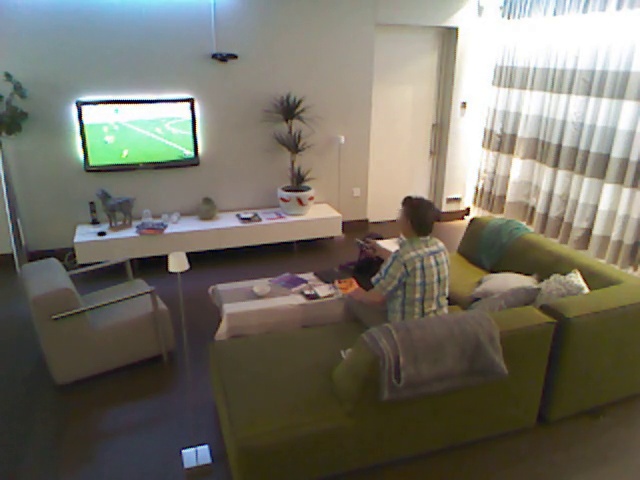}}
    \subfigure[DALLE]{\includegraphics[width=.4\linewidth,height=5.2cm]{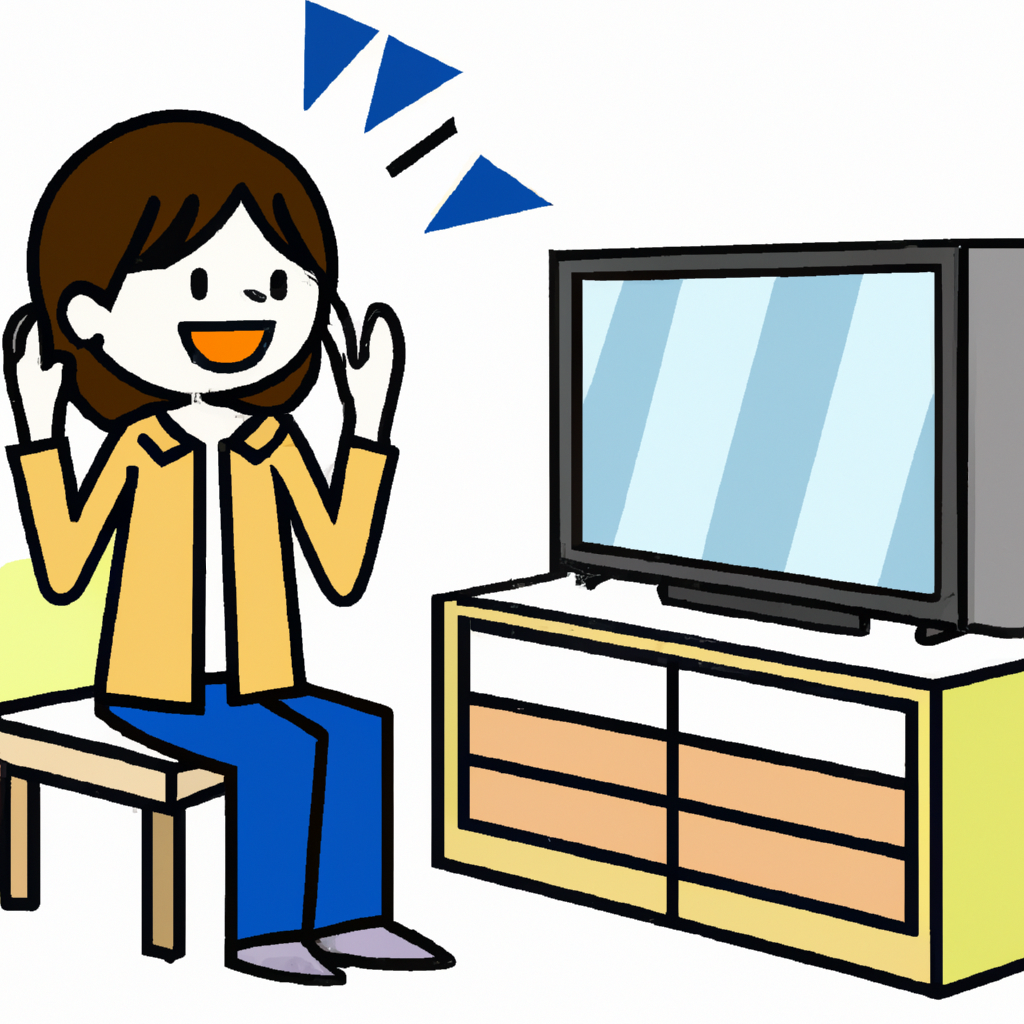}}
    \caption{Instances of the `Watching TV' action.}%
    \label{fig:tsi_example_2}
\vskip -0.2in
\end{figure}

\begin{figure}%
\vskip 0.2in
\centering
    \subfigure[TSI]{\includegraphics[width=.4\linewidth,height=5.2cm]{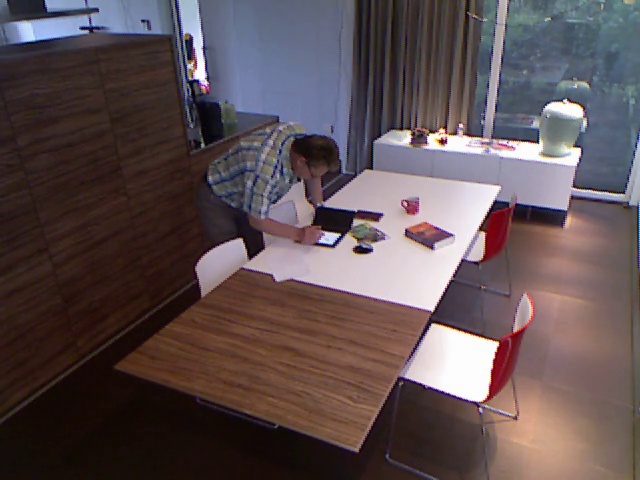}}
    \subfigure[DALLE]{\includegraphics[width=.4\linewidth,height=5.2cm]{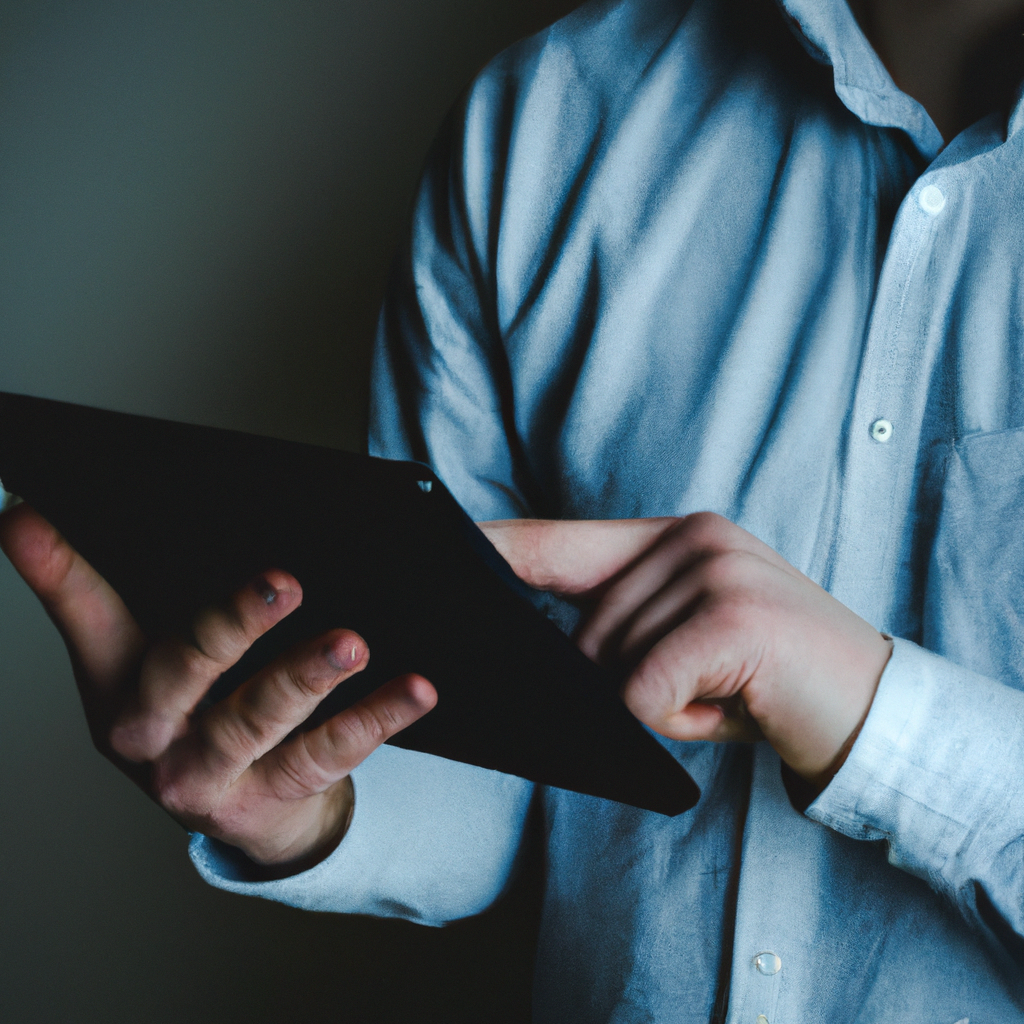}}
    \caption{Instances of the `Use Tablet' action.}%
    \label{fig:tsi_example_3}
\vskip -0.2in
\end{figure}

\begin{figure}%
\vskip 0.2in
\centering
    \subfigure[TSI]{\includegraphics[width=.4\linewidth,height=5.2cm]{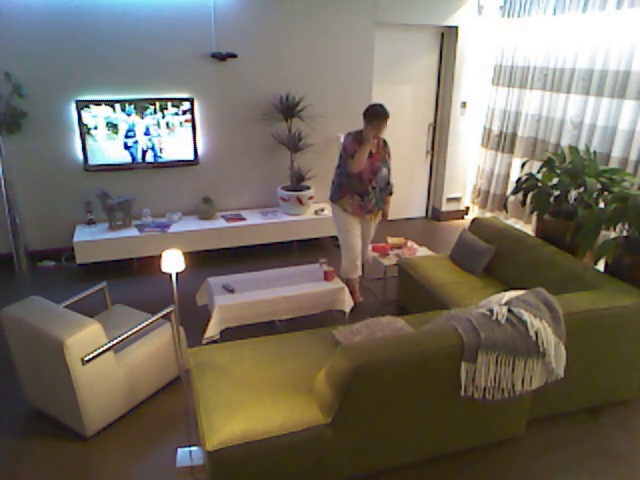}}
    \subfigure[DALLE]{\includegraphics[width=.4\linewidth,height=5.2cm]{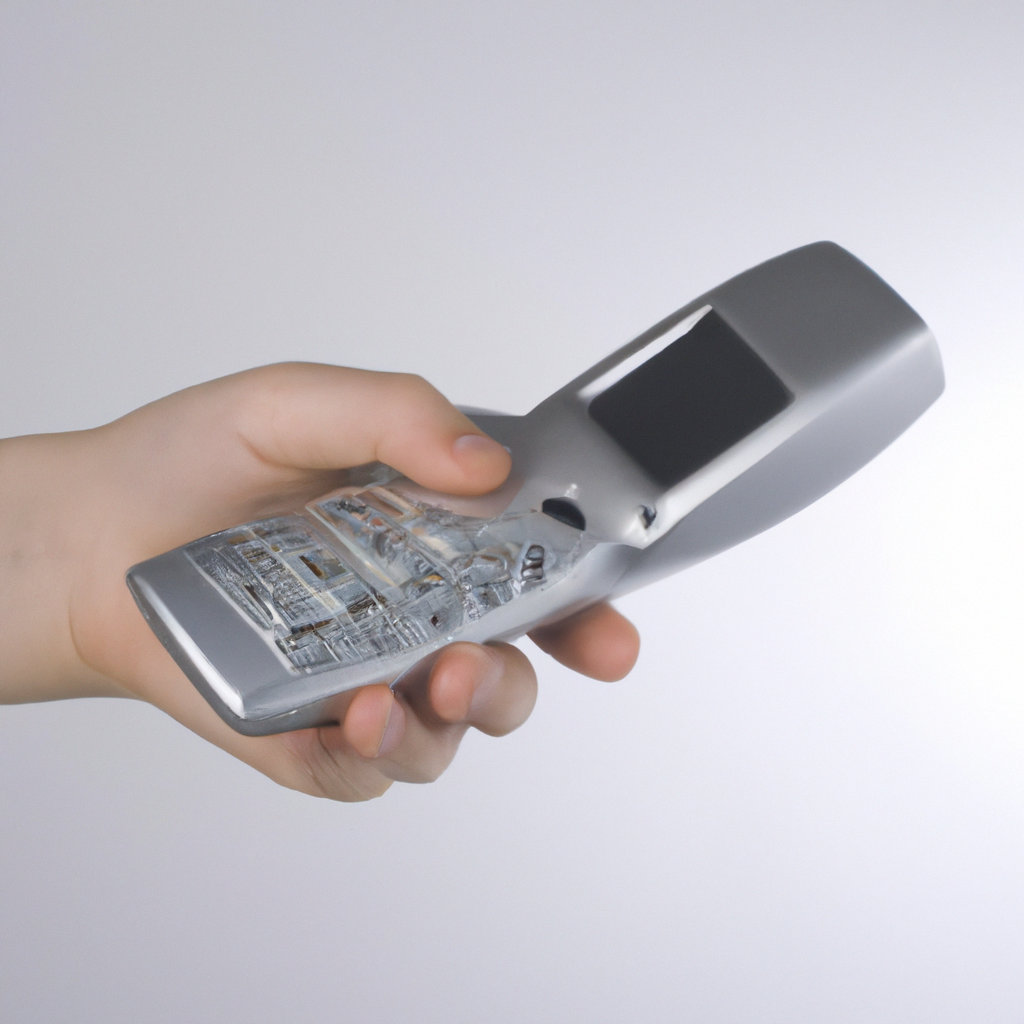}}
    \caption{Instances of the `Use  a telephone' action.}%
    \label{fig:tsi_example_4}
\vskip -0.2in
\end{figure}


\end{document}